\documentclass{article}

\usepackage{microtype}
\usepackage{graphicx}
\usepackage{subfigure}
\usepackage{booktabs} 
\usepackage{pifont}

\usepackage{hyperref}
\usepackage{fontawesome}

\usepackage{hyperref}
\usepackage{xspace}
\usepackage{color, colortbl}
\definecolor{Highlight}{rgb}{0.94,1,0.94}
\definecolor{Highlight2}{rgb}{0.82, 0.84, 0.75}
\definecolor{Highlight3}{rgb}{0.92,0.94,0.85}
\usepackage{fontawesome}

\definecolor{gray1}{rgb}{0.6, 0.6, 0.6}
\definecolor{gray2}{rgb}{0.75, 0.8, 0.99}
\definecolor{gray3}{rgb}{0.84,0.89,0.99}
\definecolor{gray4}{rgb}{0.93,0.95,0.99}

\usepackage{amsmath}
\usepackage{multicol}
\usepackage{multirow}
\usepackage{pifont}
\newcommand{\lowrank}{GaLore\xspace}

\usepackage{amsthm}
\newtheorem{theorem}{Theorem}[section]

\newlength\myindent
\usepackage{amsmath,amsfonts,bm}

\def\eqref#1{equation~\ref{#1}}

\def\1{\bm{1}}

\DeclareMathAlphabet{\mathsfit}{\encodingdefault}{\sfdefault}{m}{sl}
\SetMathAlphabet{\mathsfit}{bold}{\encodingdefault}{\sfdefault}{bx}{n}

\def\sR{{\mathbb{R}}}

\newcommand{\name}{\texttt{APOLLO}\xspace}
\newcommand{\namec}{\texttt{APOLLO-Mini}\xspace}

\usepackage[breakable]{tcolorbox}
\DeclareRobustCommand{\questionbox}[2][pink!20]{%
\begin{tcolorbox}[   
        breakable,
        left=0pt,
        right=0pt,
        top=0pt,
        bottom=0pt,
        colback=#1,
        colframe=#1,
        width=\columnwidth,
        arc=0pt,outer arc=0pt,
        ]
        #2
\end{tcolorbox}
}

\DeclareRobustCommand{\observationbox}[2][blue!5]{%
\begin{tcolorbox}[   
        breakable,
        left=0pt,
        right=0pt,
        top=0pt,
        bottom=0pt,
        colback=#1,
        colframe=#1,
        width=\columnwidth,
        arc=1pt,outer arc=1pt,
        ]
        #2
\end{tcolorbox}
}

\DeclareRobustCommand{\conclusionbox}[2][teal!10]{%
\begin{tcolorbox}[   
        breakable,
        left=0pt,
        right=0pt,
        top=0pt,
        bottom=0pt,
        colback=#1,
        colframe=#1,
        width=\columnwidth,
        arc=0pt,outer arc=0pt,
        ]
        #2
\end{tcolorbox}
}

\definecolor{darkgreen}{HTML}{2ca02c}

\usepackage[accepted]{mlsys2025}

\mlsystitlerunning{APOLLO: SGD-like Memory, AdamW-level Performance}

\begin{document}

\twocolumn[
\mlsystitle{\includegraphics[width=0.03\textwidth]{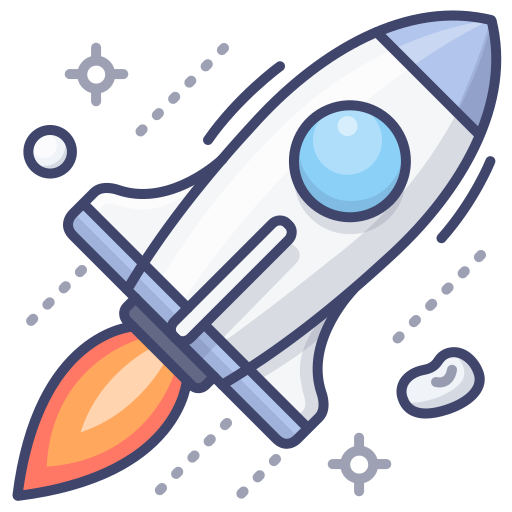} APOLLO: 
SGD-like Memory, AdamW-level Performance}



\mlsyssetsymbol{equal}{*}
\mlsyssetsymbol{coadv}{\ding{61}}

\begin{mlsysauthorlist}
\mlsysauthor{Hanqing Zhu}{equal,to,goo}
\mlsysauthor{Zhenyu Zhang}{equal,to}
\mlsysauthor{Wenyan Cong}{to}
\mlsysauthor{Xi Liu}{goo}
\mlsysauthor{Sem Park}{goo}
\mlsysauthor{Vikas Chandra}{goo}
\mlsysauthor{Bo Long}{goo}
\\
\mlsysauthor{David Z. Pan}{coadv,to}
\mlsysauthor{Zhangyang Wang}{coadv,to}
\mlsysauthor{Jinwon Lee}{coadv,goo}
\end{mlsysauthorlist}

\mlsysaffiliation{to}{Department of Electrical and Computer Engineering, The University of Texas at Austin}
\mlsysaffiliation{goo}{AI at Meta; Work was done during Hanqing's internship at Meta}

\mlsyscorrespondingauthor{David Z, Pan}{dpan@ece.utexas.edu}
\mlsyscorrespondingauthor{Zhangyang Wang}{atlaswang@utexas.edu}
\mlsyscorrespondingauthor{Jinwon Lee}{jinwonl@meta.com}

\mlsyskeywords{Machine Learning, MLSys}
\definecolor{myteal}{RGB}{128, 0, 128}

\vskip 0.2in

\begin{abstract}
Large language models (LLMs) demonstrate remarkable capabilities but are notoriously memory-intensive during training, particularly with the popular AdamW optimizer. This memory burden often necessitates using more or higher-end GPUs or reducing batch sizes, limiting training scalability and throughput, respectively. To address this, various memory-efficient optimizers have been proposed to reduce optimizer memory usage. However, they face key challenges: (i) reliance on costly SVD operations (e.g., GaLore, Fira); (ii) significant performance trade-offs compared to AdamW (e.g., Flora); and (iii) still substantial memory overhead of optimization states in order to maintain competitive performance (e.g., 1/4 rank in GaLore, and full-rank first momentum in Adam-mini).

In this work, we investigate the redundancy in AdamW's learning rate adaption rule and identify that it can be coarsened as a structured learning rate update (channel-wise or tensor-wise).
Based on this insight, we propose a novel approach, \textit{\underline{Ap}proximated Gradient Scaling for Mem\underline{o}ry Efficient \underline{LL}M \underline{O}ptimization} (\textbf{\name}), which approximate the channel-wise learning rate scaling with an auxiliary low-rank optimizer state based on pure \textit{random projection}.
The structured learning rate update rule makes \name highly tolerant to further memory reduction with lower rank, halving the rank while delivering similar pre-training performance.
We further propose an extreme memory-efficient version, \namec, which utilizes tensor-wise scaling with only a rank-1 auxiliary sub-space, achieving \textbf{SGD-level memory cost} but superior pre-training performance than Adam(W).

We conduct extensive experiments across different model architectures and tasks, showing that \name series performs \textbf{generally on-par with, or even better than Adam(W)}. Meanwhile, \name achieves \textbf{even greater memory savings than GaLore}, by almost eliminating the optimization states in AdamW.
These savings translate into significant system benefits:
(1) \textit{\textbf{Enhanced Throughput:}} \name and \namec achieve around 3$\times$ throughput on an 8$\times$A100-80GB setup compared to AdamW by fully utilizing memory to support 4$\times$ larger batch sizes.
(2) \textit{\textbf{Improved Model Scalability:}} \namec \textit{for the first time} enables pre-training LLaMA-13B model with naive DDP on A100-80G without requiring other system-level optimizations.
(3) \textit{\textbf{Low-End GPU Friendly Pre-training:}} Combined with quantization, the \name series \textit{for the first time} enables the training of LLaMA-7B from scratch on a single GPU using less than 12 GB of memory. 
\vspace{-5pt}
\begin{center}
\small
    \textbf{Website:}
    \href{https://zhuhanqing.github.io/APOLLO/}{\faGlobe\ \textcolor{myteal}{\texttt{https://zhuhanqing.github.io/APOLLO}}} \\
    \textbf{Code:} \href{https://github.com/zhuhanqing/APOLLO}{\faGithub\ \textcolor{myteal}{\texttt{https://github.com/zhuhanqing/APOLLO}}}
\end{center}
\vspace{-10pt}

\end{abstract}

]

\printAffiliationsAndNotice{\mlsysEqualContribution \mlsysEqualAdvising} 

\section{Introduction}
\label{sec:intro}
\begin{figure*}[!htb]
    \centering
    \includegraphics[width=1 \linewidth]{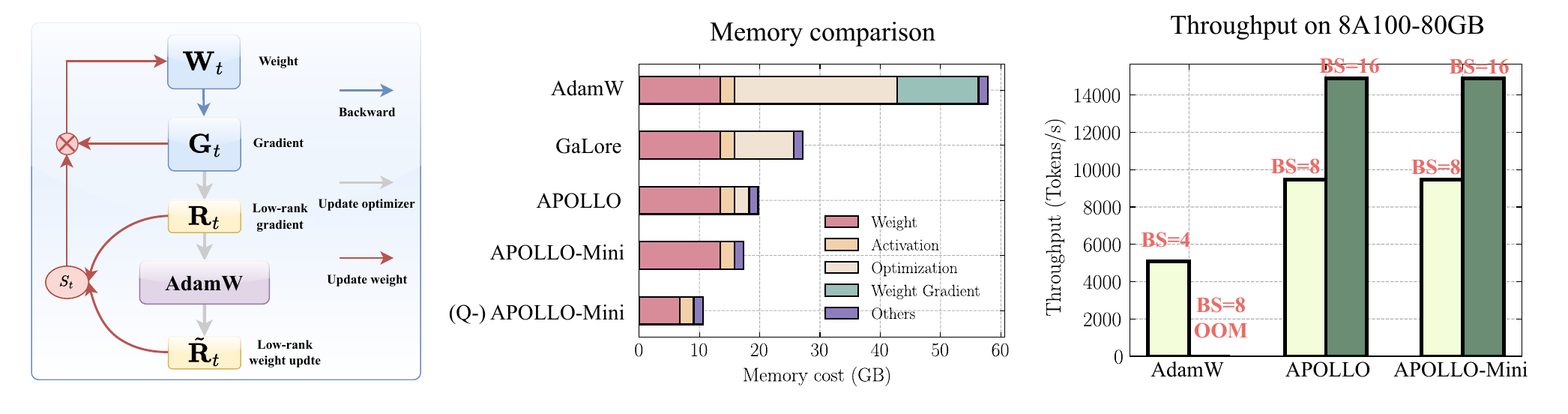}
    \vspace{-6mm}
    \caption{(Left) Overview of our \name optimizer; (Middle) Memory breakdown comparison for a single batch size, where both GaLore and our method employ the layer-wise gradient update strategy~\cite{lv2023full}. The (Q-) prefix indicates the integration of INT8 weight quantization, as utilized in~\cite{zhang2024q}; (Right) End-to-end training throughput on 8 A100-80GB GPUs. 
    }
    \label{fig:teasor}
    \vspace{-10pt}
\end{figure*}

\begin{figure}
    \centering
    \includegraphics[width=0.98\linewidth]{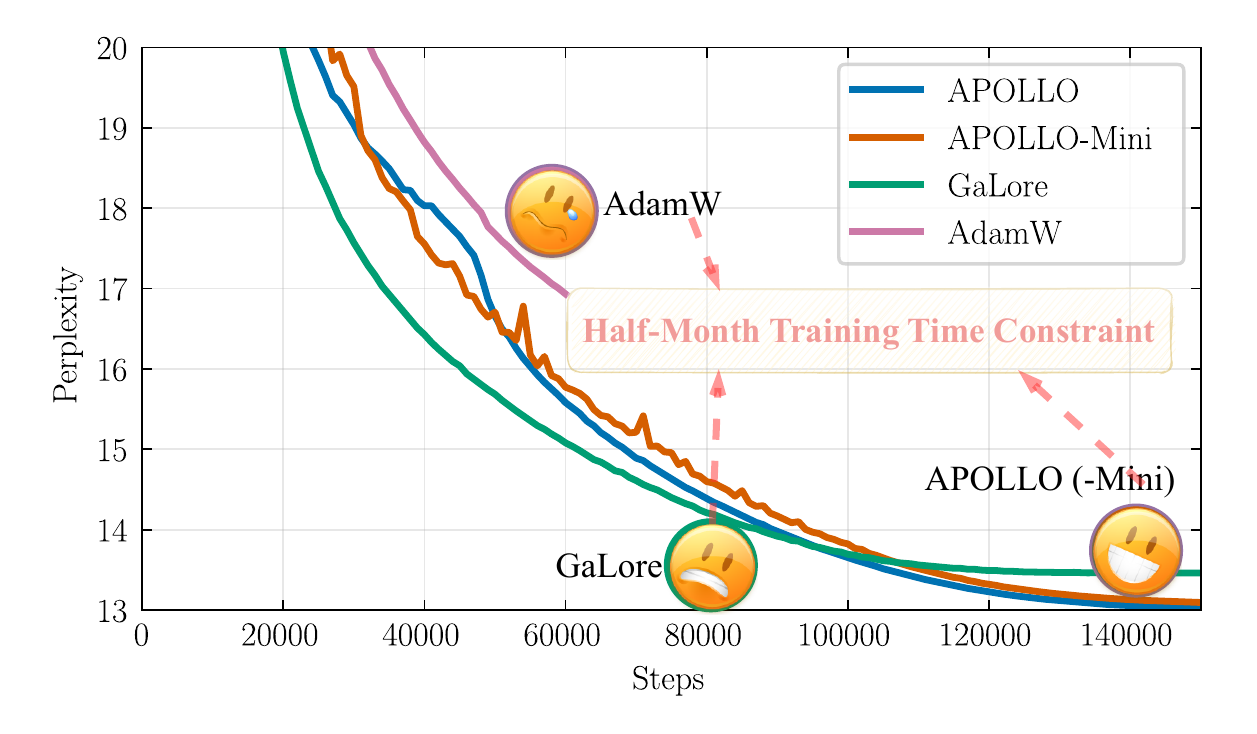}
    \vspace{-4mm}
    \caption{Comparison of Validation perplexity on LLaMA-7B.}
    \vspace{-4mm}
    \label{fig:7b-curve}
\end{figure}
Large Language Models (LLMs) have achieved remarkable progress across various domains~\cite{brown2020language, kocon2023chatgpt, dubey2024llama}, largely due to substantial increases in model size, now reaching billions of parameters. 
Training these high-dimensional models demands robust optimization techniques, with the Adam(W) optimizer~\cite{kingma2014adam,loshchilov2017decoupled} emerging as the de-facto standard for stabilizing LLM training~\cite{zhang2024transformers} by tracking both first-order and second-order moments. 
Despite its effectiveness, Adam(W) incurs significant memory overhead, as maintaining both moments effectively triples the memory required relative to the model’s parameter size. 
This results in excessive memory consumption for the optimizer,
even with a single batch.
For instance, training a LLaMA-7B model from scratch requires at least 58 GB of memory, with 28 GB devoted to AdamW’s optimizer states~\cite{zhao2024galore}. 
For larger models like GPT-3, with 175 billion parameters, memory demands reach 700 GB for the model alone, leading to a staggering 1.4 TB requirement for AdamW’s optimizer states.

This excessive optimizer memory usage poses significant challenges in training large-scale LLMs. It compels the community to either use more and higher-end GPUs, or to reduce batch sizes. However, scaling training clusters introduce highly non-trivial communication and infrastructure overheads~\cite{jiang2024megascale}; smaller batch sizes come at the cost of training throughput; and high-end GPUs are often inaccessible to researchers with limited resources.

Significant research efforts have focused on addressing the high memory costs of training LLMs. One approach involves reducing the parameter volume through methods such as designing smaller-scale LLMs~\cite{liu2024mobilellm, tang2024rethinking}, employing sparse model training~\cite{liu2022more, thangarasa2023spdf}, and leveraging low-rank adaptation~\cite{hu2021lora}. Although these techniques effectively reduce memory usage, they restrict the optimization space of model parameters and often result in performance trade-offs~\cite{biderman2024lora}, particularly during pretraining~\cite{lialin2023relora}.

Another avenue of research focuses on designing memory-efficient optimizers that reduce memory consumption while achieving performance on par with Adam(W). This includes exploring redundancy in optimizer states~\cite{zhang2024adam} and leveraging low-rank properties~\cite{zhao2024galore, chen2024fira}. Among low-rank-based methods, GaLore \cite{zhao2024galore} stands out, enabling full-parameter training of LLMs by performing low-rank gradient updates through Singular Value Decomposition (SVD). Fira \cite{chen2024fira} enhances GaLore by incorporating the error residual between the full-rank gradient and its low-rank approximation, effectively simulating full-rank updates. LDAdam \cite{robert2024ldadam} also integrates a generalized error feedback mechanism to explicitly account for the compression of gradient and optimizer states. 

However, the periodic updates to the gradient subspace via SVD (e.g., every 200 iterations)
incurs a computational cost of $O(mn^2)$, prohibitive when the matrix dimensions, $m$ and $n$, are large, as is the case with LLMs. 
For instance, in the case of the LLaMA-7B model, a single subspace update can take approximately 10 minutes, whereas inference only takes seconds. This substantial overhead significantly reduces training throughput, as demonstrated in Fig.~\ref{fig:fira:throuput}. Several GaLore variants, therefore, explore replacing the SVD projection with online PCA \cite{liang2024memory} or random projections \cite{he2024subspace}, yielding provable convergence. 

In contrast, the recently proposed Adam-mini~\cite{zhang2024adam} demonstrates that a block-wise second moment $\mathbf{V}$ suffices for learning rate adjustments, offering an orthogonal and more efficient alternative. However, achieving performance on par with AdamW requires careful handling of different model components to maintain compatibility with its optimization dynamics.

In this paper, we effectively integrate the two idea streams of  \textit{low-rank approximation} and \textit{optimizer state redundancy}, introducing a unified framework that achieves significant memory savings (below GaLore and its variants \& close to SGD) while maintaining or surpassing the performance of Adam(W). Our \textbf{key observation} is that AdamW's element-wise learning rate update rule can be effectively restructured into a channel-wise or even tensor-wise format, where each channel or tensor shares the same gradient scaling factor. 
To enable this structured gradient scaling, we introduce a memory-efficient approximation for the scaling factors using an auxiliary optimizer state, requiring only lower-dimensional gradient information as input. 
This significantly reduces memory usage by leveraging a compressed representation of the gradient information. 
Additionally, we eliminate the need for costly SVD-based low-rank projections by adopting an \textbf{SVD-free} method based on random projections. We show that a much lower rank, or even a \textbf{rank-1 approximation}, is sufficient to capture the structured gradient scaling factors effectively. This innovation allows for a simpler and more efficient training process without compromising performance. Our new memory-efficient optimizer for LLM training, named \textit{\underline{Ap}proximated Gradient Scaling for Mem\underline{o}ry Efficient \underline{LL}M \underline{O}ptimization} (\textbf{\name}), not only achieves better performance than AdamW but also delivers greater memory savings compared to GaLore.

Our key contributions are as follows:
\begin{itemize}
    \vspace{-1.5mm}
    \item \textbf{Structured Learning Rate Update for LLM Training}: We show that structured learning rate updates, such as channel-wise or tensor-wise scaling, are sufficient for LLM training. This addresses redundancy in AdamW's element-wise learning rate update rule and reduces computational overhead.
    \vspace{-1.5mm}
    \item \textbf{Approximated Channel-wise Gradient Scaling in a Low-Rank Auxiliary Space (\name)}: We propose a practical and memory-efficient method to approximate channel-wise gradient scaling factors in an auxiliary low-rank space using pure random projections. \name achieves superior performance to AdamW, even with lower-rank approximations, while maintaining excellent memory efficiency.
    \vspace{-1.5mm} 
    \item \textbf{Minimal-Rank Tensor-wise Gradient Scaling (\namec)}: For extreme memory efficiency, we introduce \namec, which applies tensor-wise gradient scaling using only a rank-1 auxiliary sub-space. \namec achieves SGD-level memory costs while outperforming AdamW, demonstrating the effectiveness of our approach.\vspace{-1.5mm}    
\end{itemize}
We demonstrate the efficacy of the \name series in both pre-training and fine-tuning scenarios.
In pre-training, across a range of LLaMA model sizes from 60M to 7B parameters, \name and \namec consistently outperform AdamW, achieving up to a 2.8 reduction in validation perplexity while significantly reducing memory overhead by eliminating nearly all optimizer states. In fine-tuning, \name and \namec achieve performance on par with full fine-tuning.
Beyond these performance gains, the \name series offers practical \textbf{system-level} advantages, including:
(i) \textbf{3× better throughput} on pre-training LLaMA 7B ( Fig.\ref{fig:teasor} (right) and the 7B experiments in Fig.\ref{fig:7b-curve}); (2) \textbf{Extreme training memory savings}. By combining \namec with weight quantization, we \underline{\textbf{set a new record for memory efficiency}}: pre-training a LLaMA 7B model requires only \textbf{12GB} of memory (Fig.~\ref{fig:teasor} (middle)). More information can be found in Section~\ref{sec:memory}. These results establish \name and \namec as highly efficient and scalable solutions for both pre-training and fine-tuning of LLMs, offering compelling improvements in performance, memory usage, and throughput.

\section{Related Work}

\subsection{Algorithm-Level Memory-Efficient Training}
\vspace{-2mm}

Numerous algorithmic improvements have been introduced to tackle the substantial memory overhead in training LLMs. One category focuses on reducing the number of trainable parameters to save memory costs. This includes approaches such as developing high-quality, small-scale models~\cite{liu2024mobilellm, tang2024rethinking}, introducing sparsity during training~\cite{liu2022more, thangarasa2023spdf}, and implementing low-rank adaptation~\cite{hu2021lora}. While these methods are effective at reducing memory usage, they often fall short in achieving comparable performance, especially in pre-training scenarios for large models.

Another avenue of research targets advancements in optimizers, as exemplified by works such as GaLore~\cite{zhao2024galore}, Fira~\cite{chen2024fira}, Flora~\cite{hao2024flora}, Adam-mini~\cite{zhang2024adam}, GaLore-mini~\cite{huang2024galore}, LDAdam~\cite{robert2024ldadam}, GoLore~\cite{he2024subspace}, and LoQT~\cite{loeschckeloqt}. These approaches have made notable progress but still face significant challenges. Some methods rely on computationally expensive SVD operations (e.g., GaLore and Fira), although recent research shows that random projections can effectively compress gradients during later training stages while still requiring SVD early on~\cite{he2024subspace}. Others either exhibit noticeable performance gaps compared to AdamW, or demand substantial memory overhead to maintain competitive performance, as seen in GaLore's 1/4 rank requirement and Adam-mini's reliance on full-rank first momentum.

In contrast, \name achieves efficient memory usage entirely without relying on SVD while delivering performance that matches or even surpasses AdamW. Moreover, our extreme variant, \namec, drives memory costs down to SGD levels while maintaining or exceeding the performance of AdamW, setting a new benchmark for memory-efficient optimization.

\subsection{System-Level Memory Efficiency Optimization}
\vspace{-2mm}
Several system-level techniques have been developed to reduce memory usage in LLM training \citep{chen2016training, ren2021zero}. Activation checkpointing~\cite{chen2016training} recomputes activations during backward instead of storing them for the whole training iteration, reducing memory requirements. 
Quantization \citep{dettmers2024qlora} reduces memory requirements by utilizing lower bit data formats.
Memory offloading \citep{zhang2023g10, ren2021zero} reduces GPU memory consumption by leveraging non-GPU memory. Our method, \name, is orthogonal to these system-level optimizations and can be seamlessly integrated to achieve greater memory efficiency.
Furthermore, by eliminating the need for SVD, \name is more system-friendly,  requiring only a cheap general matrix multiplication to complete the projection step.

\section{Coarsened Learning Rate Update Rule Is Enough for LLMs}
In this section, 
we first revisit the Adam(W)~\cite{kingma2014adam,loshchilov2017decoupled} and reformulate it as an adaptive learning rate algorithm without explicit momentum term (Section~\ref{sec:reform_adam}).
Then, we propose that the element-wise learning rate update rule can be coarsened with a structured channel-wise learning rate adaptation strategy, with even slightly better model performance by empirical verification.

\subsection{Reformulating AdamW as a Pure Adaptive Learning Rate Algorithm}
\label{sec:reform_adam}

\paragraph{Vanilla AdamW update rule.} 
AdamW has established itself as the go-to optimizer for Transformer training, leveraging both \textbf{first moment} (the mean of past gradients) and \textbf{second moment} (the variance of past gradients) to adjust updates. This dual momentum-based approach has proven superior to purely first-order optimizers like SGD~\cite{zhang2024transformers}. Disregarding weight decay, the vanilla AdamW update rule is as follows:

At time step \( t \), given a weight matrix \( \mathbf{W} \in \mathbb{R}^{m \times n} \) ($m \leq n$) with gradient \( \mathbf{G}_t = -\nabla_{\mathbf{W}} \phi_t(\mathbf{W}_t) \), the standard AdamW update rule is defined as:
\begin{equation}
\label{eq:adam}
    \mathbf{W}_{t+1} = \mathbf{W}_t - \eta \cdot \tilde{\mathbf{G}_t}, \quad \tilde{\mathbf{G}_t} = \frac{\mathbf{M}_t}{\sqrt{\mathbf{V}_t} + \epsilon}
\end{equation}
Here, \( \eta \) is the learning rate and \( \epsilon \) is a small constant for numerical stability. 
The first and second moment, \(\mathbf{M}_t \) and \(\mathbf{V}_t \), are computed as exponentially weighted averages:
\[\mathbf{M}_t = \beta_1 \mathbf{M}_{t-1} + (1-\beta_1) \mathbf{G}_t\]
\[\mathbf{V}_t = \beta_2 \mathbf{V}_{t-1} + (1-\beta_2) \mathbf{G}^2_t\]

where $\beta_1, \beta_2 \in [0, 1)$ are the exponential decay rates.

\paragraph{Viewing AdamW as an adaptive learning rate algorithm without momentum.} The above update rule in \eqref{eq:adam} can then be reformulate as an element-wise \textbf{gradient scaling rule} with a gradient scaling factor $\mathbf{S}=\frac{\tilde{\mathbf{G}_t}}{\mathbf{G}_t} \in \sR^{m \times n}$ over the raw gradient \( \mathbf{G}_t \), \textit{i.e.},
\begin{equation}
    \label{eq:reform_adam}
   \mathbf{W}_{t+1} = \mathbf{W}_t - \eta \cdot \frac{\tilde{\mathbf{G}_t}}{\mathbf{G}_t} \cdot \mathbf{G}_t     
\end{equation}

In other words, the effectiveness of AdamW can be viewed as the result of a \textbf{variance-aware learning rate schedule} per element in raw gradient $\mathbf{G}_t$ using the corresponding element in $\mathbf{S}$, where elements with higher variance in \( \mathbf{V}_t \) are scaled down to reduce unstable updates. While this reformulation is very straightforward, it paves the way for subsequent analysis. It also provides a convenient strategy to analyze other momentum algorithms through ``SGD-like" lens (e.g., all reduced to adaptive SGD learning rates).

\subsection{Coarsening Element-wise Learning Rate Adjustment in a Structured Manner}

While the element-wise learning rate update rule in AdamW is effective, it can be \textbf{overly sensitive to noisy gradients} in specific parameters, especially in high-dimensional models like large language models (LLMs).
Recent work, such as Adam-mini~\cite{zhang2024adam}, proposes grouping parameters into blocks and applying a \textbf{block-wise learning rate adjustment} to reduce memory usage while maintaining on-par performance as Adam(W).
However, the block-wise approach in Adam-mini~\cite{zhang2024adam} requires carefully chosen block sizes for different modules in Transformers and only achieves memory savings for the second moments, leaving the first moment memory unaffected.

\paragraph{A more structured learning rate update rule.} Inspired by findings of optimizer redundancy, we propose an effective simplification by coarsening the element-wise adaptive learning rate rule in \eqref{eq:reform_adam} into 
a \textbf{structured channel-wise adaptation}. We group parameters based on the larger dimension of the weight tensors. 
The element-wise scaling factor \( \mathbf{S} = \frac{\tilde{\mathbf{G}}_t}{\mathbf{G}_t} \) is then simplified into a \textbf{channel-wise} format
, \( s \in \mathbb{R}^{1 \times n} \), where each element $s_j$ for each channel \( j \) is:
\begin{equation}
    \label{eq:col_scale}
    s_j = \frac{\| \tilde{\mathbf{G}}_t[:, j] \|_{2}}{\| \mathbf{G}_t[:, j] \|_{2}}
\end{equation}
where \( \| \cdot \|_{2} \) denotes the $\ell_2$ norm. Then, the final gradient scaling rule becomes \(\tilde{\mathbf{G}}_t = \mathbf{S} \cdot \mathbf{G}_t = \mathbf{G}_t \cdot \text{diag}(s).
\)

As suggested by the authors of concurrent work Fira~\cite{chen2024fira},
our channel-wise approximation of the element-wise gradient scaling rule \(\frac{\tilde{\mathbf{G}_t}}{\mathbf{G}_t}\) shares a similar format the scaling factor in the Fira, which is used for normalizing the error residual between low-rank GaLore and full-rank gradients.
While our approach shares a similar mathematical form, \textit{as being a straightforward computation of \(\ell_2\)-norm ratios}, it originates from a fundamentally distinct perspective.
We argue that the element-wise gradient scaling rule in \eqref{eq:reform_adam} is unnecessarily fine-grained and can be effectively replaced with structured \textit{channel-wise} or \textit{tensor-wise} adaptation. 
In contrast, Fira seeks to normalize the error residual between low-rank GaLore updates and full-rank updates based on the observation that channel-wise gradient norm ratios between low-rank and full-rank optimizers are inherently similar.
Our method, however, establishes a different finding: the low-rank approximated channel-wise gradient scaling factor, \(\frac{\tilde{\mathbf{G}_t}}{\mathbf{G}_t}\), follows a predictable ratio of \(\sqrt{r/n}\) (see Theorem~\ref{theorm:main}) compared to full-rank optimization, which differs fundamentally from Fira's observations.

\paragraph{Empirical validation.} We first empirically explore the effectiveness of the proposed update rule where we compared the training loss of the original element-wise learning rate adaptation with our proposed channel-wise adaptation on a LLaMA-130M model. 
As shown in Figure~\ref{fig: loss-motivation}, both approaches achieve similar loss reduction over training steps, 
demonstrating that the structured adaptation effectively maintains performance. 
In fact, the channel-wise adaptation achieves slightly better perplexity 24.43 (AdamW: 25.08), further supporting the effectiveness of our approach.
However, we notice that our channel-wise learning rate adaption ({\textcolor{orange}{orange curve}) shows a significant spike at the early stage, which is due to the unstable gradient at the early stage. 
Instead of applying the vanilla gradient clipping method, we use the Norm-growth Limiter (NL) in ~\cite{chen2024fira} to limit the consecutive gradient growth, as it is shown slightly more effective than gradient clipping:
\begin{equation}{\label{Fira_limiter}}
\text{if } \frac{\displaystyle || \tilde{\mathbf{G}}_t ||}{\displaystyle || \tilde{\mathbf{G}}_{t-1} ||} > \gamma \text{ then } \tilde{\mathbf{G}}_t \leftarrow  \frac{\tilde{\mathbf{G}}_t }{\displaystyle || \tilde{\mathbf{G}}_t ||} \cdot \gamma \displaystyle || \tilde{\mathbf{G}}_{t-1} ||
\end{equation}
where $\gamma$ is a threshold to ensure that the rate of gradient growth remains controlled. 
This approach limits the magnitude of gradient norm increases, particularly for the unstable gradients in the early stages, thereby preventing loss spikes ({\textcolor{darkgreen}{green curve}), leading to further better perplexity 24.11.
We, by default, use the NL in our method and set $\gamma=1.01$.

\observationbox{
\textit{Takeaways \ding{172}: A structured learning rate update is sufficient for LLM training.}
}

This observation suggests that effective optimization can be achieved by applying adaptive learning rates at a coarser granularity, such as channel-wise, rather than at the element-wise level. This insight forms the basis for the memory-efficient methods we propose in the next section.

\begin{figure}[h] 
\centering 
\includegraphics[width=0.95\linewidth]{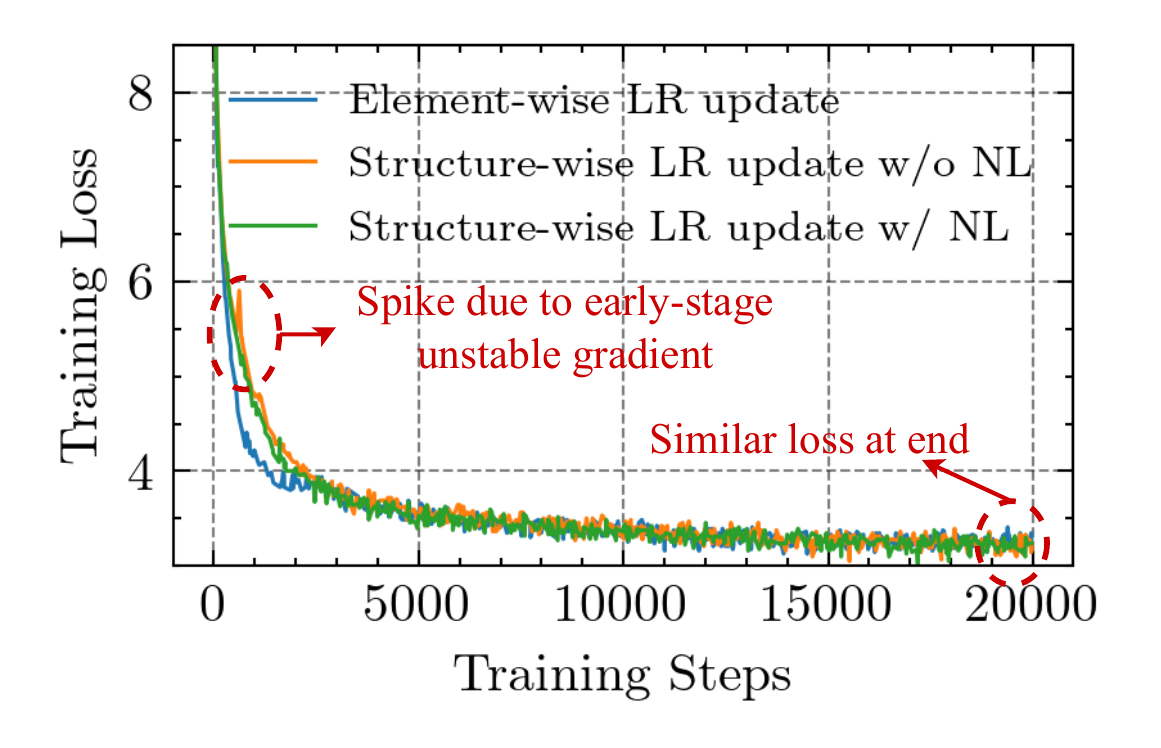}
\vspace{-4mm}
\caption{Training loss comparison between Element-wise and Channel-wise Learning Rate (LR) Adaptations with or without norm limiter (NL) on the LLaMA-130M model.}
\label{fig: loss-motivation} 
\vspace{-10pt}
\end{figure}

\section{\name: Approximated Gradient Scaling for Memory Efficient LLM Optimization}

\paragraph{From observation to practical Benefit.} 
While coarsening gradient scaling factors is effective, it does not inherently reduce optimizer memory usage. 
Computing structured gradient scaling factors still requires access to the full optimizer states $\mathbf{M}_t$ and $\mathbf{V}_t$, which consume significant memory. 
This brings us to a critical question:
\questionbox{ 
\textit{Question \ding{172}: Can structured learning rate adaptation be converted into practical, memory-efficient optimization?} 
}

\subsection{\name: Approximate Structural Gradient Scaling for LLM Optimization}
\label{subsec:apollo_channel_wise}

\definecolor{highlightcolor}{rgb}{0.9, 0.9, 0.9}  %
\begin{algorithm}[tb]
   \caption{AdamW with \name/\namec}
   \label{alg:appolo_adam}
 \begin{algorithmic}
   \STATE {\bfseries Input:} A weight matrix $\textbf{W} \in \mathbb{R}^{m \times n}$ with $m \leq n$. Step size $\eta$, scale factor $\alpha$, decay rates \{$\beta_1, \beta_2$\}, weight decay $\lambda$, rank $r$, subspace update frequency $T$.

   \STATE \textbf{Initialize}: $t \gets 0$
   \REPEAT
   \STATE \textcolor{gray}{\# Step 1: Calculate gradient into low rank space.}
   \STATE $\mathbf{G}_t \in \mathbb{R}^{m \times n} \gets - \nabla_{\mathbf{W}} \phi_t(\mathbf{W}_t)$ 
   \IF{$t \bmod T = 0$}
   \STATE \colorbox{highlightcolor}{$\mathbf{P}_t \gets \mathcal{N}_{seed}(0, 1/ r)$}
   \STATE seed $\gets$ an independent new random seed
   \ENDIF
   \STATE $\mathbf{R}_t \gets \mathbf{P}_{t} \mathbf{G}_t$
   
   \STATE \textcolor{gray}{\# Step 2: Obtain low rank optimization states, $\mathbf{M}_t, \mathbf{V}_t$.}
   \STATE $\mathbf{M}_t^{R}, \mathbf{V}_t^{R} \gets \mathrm{AdamW}(\mathbf{R}_t, \beta_1, \beta_2, \lambda=0)$
   \STATE $\tilde{\mathbf{R}}_t \gets \mathbf{M}^{R}_t / (\sqrt{\mathbf{V}^{R}_t} + \epsilon)$
   \STATE \textcolor{gray}{\# Step 3: Obtain approximated gradient scaling factor.}
   \IF{\name}
    \STATE \colorbox{highlightcolor}{$\mathbf{S}\leftarrow \text{diag}(s^R_0, s^R_1, ..., s^R_m)$} \COMMENT{$s^R_i =  \frac{\| \tilde{\mathbf{R}}_t[:, j] \|_2}{\|\mathbf{R}_t[:, j] \|_2}$}
    \ELSIF{\namec}
    \STATE \colorbox{highlightcolor}{$\mathbf{S} \leftarrow s^R $} \COMMENT{$s^R =  \frac{\| \tilde{\mathbf{R}}_t \|_2}{\|\mathbf{R}_t \|_2}$}
   \ENDIF
   \STATE \textcolor{gray}{\# Step 4: Update weight in original space.}
   \STATE \colorbox{highlightcolor}{$\mathbf{W}_t \gets \mathbf{W}_{t-1} + \eta \cdot \alpha \cdot \mathbf{G}_t \mathbf{S} - \eta \cdot \lambda \mathbf{W}_{t-1}$}
   \STATE $t \gets t + 1$
   \UNTIL{convergence criteria met}
    \STATE \textbf{return} $\mathbf{W}_T$
 \end{algorithmic}
\end{algorithm}

\subsubsection{Approximating Gradient  Scaling with an Auxiliary Low-Rank Space}
To address this challenge, we propose \name, which approximates the channel-wise gradient scaling in a compressed low-rank space rather than the original full-rank one, showing in Algorithm~\ref{alg:appolo_adam}.
Specifically, an auxiliary low-rank optimizer state is stored by taking the low-rank gradient $\mathbf{R}_t$ as input, which is computed as $\mathbf{R}_t = \mathbf{P}_t \mathbf{G}_t \in \sR^{r\times n} $, using a pre-defined projection matrix $\mathbf{P}_t \in \sR^{r\times m}$.

The auxiliary optimizer state will only maintain the low-rank version of the first and second moments as:
\[\mathbf{M}_t^{R} = \beta_1 \mathbf{M}_{t-1}^{R} + (1-\beta_1) \mathbf{R}_t \]
\[\mathbf{V}_t^{R} = \beta_2 \mathbf{V}_{t-1}^{R} + (1-\beta_2) \mathbf{R}^2_t \]
These low-rank moments, $\mathbf{M}_t^R$ and $\mathbf{V}_t^R$, are then 
converted into a lightweight, channel-wise scaling factors:
\begin{equation}
    \begin{aligned}
    \label{eq:col_scale_low_rank}
    s^R_j = \frac{\| \tilde{\mathbf{R}}_t[:, j] \|_2}{\| \mathbf{R}_t[:, j] \|_2}, \text{where} \
    \tilde{\mathbf{R}}_t = \frac{\mathbf{M}_t^R}{\sqrt{\mathbf{V}_t^R} + \epsilon}
    \end{aligned}
\end{equation}

In this way, \name estimates the channel-wise gradient scaling factor $s$ with the auxiliary low-rank optimizer state, saving memory from $2mn$ to $2nr$.
We will show later that the coarse-grained channel-wise scaling makes \name insensitive to low rank, unlike GaLore, which needs relatively high rank to retain performances (typically, one-quarter of the original dimension), leading to great memory saving. Details can be found at Section~\ref{sec:ablation}, \textit{A3}.

However, since \name operates in a compressed domain (\textit{i.e.} low-rank space), a key question remains:
\begin{center} 
\questionbox{ 
\textit{Question \ding{173}: Can the adaptive learning rate in the compressed space effectively approximate its behavior in the original space?} 
} 
\end{center}

Moreover, what type of low-rank projection method is ideal for this purpose? 
The default choice might be Singular Value Decomposition (SVD), as it captures the most informative components of the gradient. 
In fact, most existing low-rank optimizers for LLMs rely on SVD-based approximations to maintain accuracy, especially during pre-training. 
However, SVD is computationally expensive for large models and cannot be efficiently parallelized on GPUs, hindering the training process.
Therefore, we pose the following question:
\begin{center} 
\questionbox{ 
\textit{Question \ding{174}: Do we still need costly SVD to construct our compressed space?} 
} 
\end{center}

\subsubsection{\name Performs Well with Random Projection: SVD is Not Necessary.}
To answer the above questions, we first analyze the norm difference between the first moment $\textbf{M}_t^R$ in the \textit{compressed space} and the \textit{original space}, as well as similar results for the second state $\textbf{V}_t^R$ and $\textbf{V}_t$. 

We demonstrate that \textit{random projection can effectively bound the difference between the gradient scaling factor} in the \textit{compact} and \textit{original} space in \eqref{eq: diff}:
\begin{equation}
\label{eq: diff}
    \begin{aligned}
        \text{\textit{Original space:}} &\ s_j= \frac{\|\tilde{\mathbf{G}}_t[:, j]\|}{\|\mathbf{G}_t[:, j]\|}, \ \tilde{\mathbf{G}}_t = \frac{\mathbf{M}_t}{\sqrt{\mathbf{V}_t}} \\
        \text{\textit{Compact space:}} &\ s_j^R= \frac{\|\tilde{\mathbf{R}}_t^R[:, j]\|}{\|\mathbf{R}_t[:, j]\|}, \ \tilde{\mathbf{R}}_t^R = \frac{\mathbf{M}_t^R}{\sqrt{\mathbf{V}_t^R}}
    \end{aligned}
\end{equation}
with all small $\epsilon$ in the denominators removed for simplicity.

\begin{figure*}[ht]
    \centering
    \includegraphics[width=0.8\linewidth]{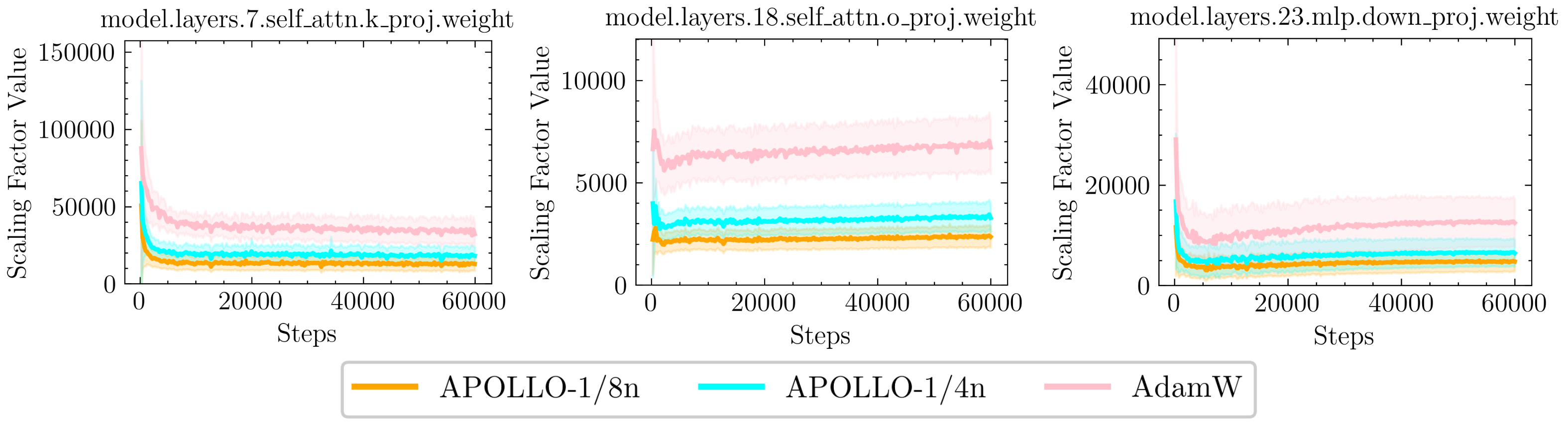}
    \caption{Visualization of the channel-wise scaling factor ratio for APOLLO with rank \(1/8 n\) and \(1/4 n\), compared with AdamW (full rank \(n\)). The empirical data aligns well with the theoretical ratios \(1 : \sqrt{2} : 2\sqrt{2}\), validating the bounds across various layer types and stages on the LLaMA-350M model. More visualization can be found at Fig.~\ref{fig:scaling_factor_comparison}.}
    \label{fig:scaling_factor_comparison_few}
\end{figure*}

\paragraph{Generating random projection matrix.}
We generate the random projection matrix $\mathbf{P}$ by sampling each element from a standard Gaussian distribution.
With high probability, projection using a random Gaussian matrix largely preserves the scaled norm from the original space based on the Johnson–Lindenstrauss lemma (JLT)~\cite{freksen2021introduction}.

\paragraph{First-order moment ratio bounding.}
We expand the computation formula of the first moment recursively, as: 
\begin{eqnarray}
    \begin{aligned}
        \mathbf{M}_t &= \beta_1 \mathbf{M}_{t-1} + (1 - \beta_1) \mathbf{G}_t \\
            &= \beta_1^t \mathbf{M}_0 + (1 - \beta_1) \sum_{k=0}^{t-1} \beta_1^k \mathbf{G}_{t-k}
    \end{aligned} \\
    \begin{aligned}
        \mathbf{M}_t^R &= \beta_1 \mathbf{M}_{t-1}^R + (1 - \beta_1) \mathbf{R}_t \\
            &= \beta_1^t \mathbf{M}_0^R + (1 - \beta_1) \sum_{k=0}^{t-1} \beta_1^k \mathbf{R}_{t-k}
    \end{aligned}
\end{eqnarray}
where $\beta_1 < 1$. 

We quantify the approximation error in the following theorem.
\begin{theorem}\label{thm:momentum}
\textbf{Approximated Channel-wise Momentum with a bound for its $\ell_2$ norm}:
$\textbf{G}_t \in \sR^{m \times n}$ is the full-rank gradient ($m \leq n$).
Let $\textbf{P}$ be a matrix of shape $\sR^{r \times m}$ where each element is independently sampled from a standard Gaussian distribution in the variance of $1/r$. 
With the projected gradient $\textbf{R}_t = \textbf{P} \textbf{G}_t$,
we have the projected gradient with a \textbf{bounded channel-wise first order moment}.
For any channel \( j \),
with probability at least $1 - 2\exp\left(-\frac{r\epsilon^2}{8}\right)$:
\[
(1-\epsilon)\|\mathbf{M}_t[:, j]\|^2 \leq \|\mathbf{M}^R_t[:, j]\|^2 \leq (1+\epsilon)\|\mathbf{M}_t[:, j]\|^2.
\]
\end{theorem}
\textit{Proof}: Please refer to Appendix~\ref{subsubsec:first_moment}.

\paragraph{Second-order moment ratio bounding.}
Similarly, the second-order moment state can be formulated as:
\[\mathbf{V}_t = (1 - \beta_2) \sum_{k=0}^{t-1} \beta_2^k \mathbf{G}_{t-k}^2\]
\[\mathbf{V}_t^R = (1 - \beta_2) \sum_{k=0}^{t-1} \beta_2^k \mathbf{R}_{t-k}^2\]
where we assume \( \mathbf{V}_0 = 0 \) (common in most initialization).

\begin{theorem}\label{thm:variance}
\textbf{Approximated channel-wise variance with a bound for its $\ell_1$ norm}:
For any channel \( j \) and time \( t \), if 
\[
r \geq \frac{8}{\epsilon^2} \log\left(\frac{2t}{\delta}\right),
\]
then with probability at least \( 1 - \delta/2 \):
\[
(1-\epsilon)\|\mathbf{V}_t[:, j]\|_1 \leq \|\mathbf{V}_t^R[:, j]\|_1 \leq (1+\epsilon)\|\mathbf{V}_t[:, j]\|_1
\],
where \( \mathbf{V}_t[:, j] \) and \( \mathbf{V}^R_t[:, j] \) are the second moments in the original and projected spaces, respectively.

\end{theorem}
\textit{Proof}: Please refer to Appendix~\ref{subsubsec:second_moment}.

\paragraph{Bounded update ratio $s^R / s$}
Now, we can bound the difference between the gradient scaling factor in the compact original space based on the theorem~\ref{thm:momentum} and theorem~\ref{thm:variance}:
\[s_j^R / s_j = \frac{\| \tilde{\mathbf{R}}_t[:, j] \|}{\| \mathbf{R}_t[:, j] \|} \cdot \frac{\| \mathbf{G}_t[:, j] \|}{\| \tilde{\mathbf{G}}_t[:, j] \|} = \frac{\| \tilde{\mathbf{R}}_t[:, j] \|}{\| \tilde{\mathbf{G}}_t[:, j] \|} \cdot \frac{\| \mathbf{G}_t[:, j] \|}{\| \mathbf{R}_t[:, j] \|}\]

For any channel \( j \), with probability \( \geq 1-\delta \):
\begin{equation}
\label{eq:main}
    \frac{\sqrt{1-\epsilon}}{1+\epsilon}\leq \sqrt{\frac{n}{r}}  \frac{s_j^R}{s_j} \leq \frac{\sqrt{1+\epsilon}}{1-\epsilon}.
\end{equation}

\textit{Proof}: Please refer to Appendix~\ref{subsubsec:all}.

Therefore, our \name method is theoretically sound with random projection and validated by the empirical results presented in the following section. Additionally, \name with SVD also performs well within our framework, yet incurs significant computational cost. We default to use random projection in \name.

The theorem suggests we should scale the gradient with the factor $\sqrt{\frac{n}{r}}$ to ensure consistent behavior as AdamW with structured learning rate update.
Hence, we add the gradient scale factor $\alpha$ in Algorithm~\ref{alg:appolo_adam}.
However, this gradient
scale factor can be combined with the learning rate, therefore set to 1 by default for \name.
When the $r$ is too small compared to $n$, as in our \namec case, which uses rank-1 space, we specifically assign the scaling factor by using $\sqrt{128}$.

Moreover, as suggested in GaLore, fixing the projection matrix is not ideal for high-dimensional LLM training; we also re-sample the projection matrix every $T$ step, which is effortlessly done by re-generating a new seed, which is set to 200 by default.

\textbf{Empiric evidence of the ratio $\sqrt{n/r}$}:
We empirically validate the theoretical bound of the scaling factor ratio \(\sqrt{n/r}\) derived in Equation~\ref{eq:main} using the LLaMA-350M model. Specifically, we compare the scaling factors of \name with ranks \(1/8n\) and \(1/4n\) against the full-rank AdamW baseline. The results, visualized in Fig.~\ref{fig:scaling_factor_comparison_few}, demonstrate that the observed ratios align closely with the theoretical predictions (\(\sim 0.35\) and \(0.5\) respectively), confirming the effectiveness of random projection in approximating gradient scaling factors within our framework.
Check more details in Appendix~\ref{subsec:empirical_evidence}.

\observationbox{
\textit{Take-away \ding{173}: \name can approximate the structured learning rate adaption with only random projection.}
}

\subsection{\namec: Achieve Extreme Memory Efficiency with Rank-1 Space}
\label{subsec:apollo_tensor_wise}

The rank $r$ plays a crucial role in balancing memory efficiency and the quality of the approximated gradient scaling factor. 
Although the coarsening learning rate update rule provides high tolerance to relatively low rank, we show our \name can reduce rank by half compared to GaLore with a slight impact on perplexity.
We still need to pay $n \times r$ memory cost for the optimizer state. 
If we can relax the rank requirement to 1, then the optimizer state cost is totally negligible. However, simply setting rank to 1 in \name using channel-wise gradient scaling doesn't work well due to rank-1 space sacrificing too much information. Details at Section~\ref{sec:ablation} \textit{A2}.
This leads to our next question:
\begin{center}
\questionbox{
    \textit{Question \ding{175}: “Can we further compress the optimizer state to SGD-level memory cost while matching or surpassing AdamW's performance?}
}
\end{center}

To address this, we introduce an extremely memory-efficient \name variant, called \namec, which coarsens the scaling factor into a \textit{tensor-wise scaling factor} to reduce variance during gradient scaling estimation in a rank-1 space.
The scaling factor is computed as:
$s = \frac{\| \tilde{R}_t \|_2}{\| R_t \|_2}$. Moreover, we find the \textit{tensor-wise scaling factor} estimated by rank-1 space is typically smaller than the one estimated with a larger rank, which can be theoretically justified by the theorem in ~\eqref{eq:main}.
Hence, we heuristically set a gradient scale factor $\alpha$ like GaLore to avoid performance degradation, default to set as $\sqrt{128}$.

\subsection{Savings and cost analysis}

Table~\ref{Comparison} provides a detailed comparison of memory and computational trade-offs among various memory-efficient training methods, including \namec, \name, Fira~\cite{chen2024fira}, GaLore~\cite{zhao2024galore}, and Flora~\cite{hao2024flora}. Notably, \name can purely implemented with random projection, thereby avoiding the costly SVD operations required by Fira and GaLore.

\begin{table}[htb]
\caption{Detailed comparison between Fira, GaLore, Flora, \name, and \namec. 
Denote $\mathbf{W}_t \in \mathbb{R}^{m \times n}$ ($m \leq n$), rank $r$.
\name series has a constant 2 due to storing random seed and gradient norm used for norm-worth limiter.
}
\label{Comparison}
\begin{center}
\resizebox{\linewidth}{!}{ 
\begin{tabular}{l|>{\columncolor{gray2}}c>{\columncolor{gray4}}c|ccc}
\toprule
   & \begin{tabular}[c]{@{}c@{}}\name\\ \texttt{-Mini}\end{tabular} & \name & Fira & GaLore & Flora \\
\midrule
     Weights & $mn$ & $mn$ & $mn$ & $mn$ & $mn$ \\
     Optimizer States & \begin{tabular}[c]{@{}c@{}}$2n$\\ $+ 2$\end{tabular} & \begin{tabular}[c]{@{}c@{}}$2nr$\\ $+ 2$\end{tabular} & \begin{tabular}[c]{@{}c@{}}$2nr$\\ $+ mr + 1$\end{tabular} & \begin{tabular}[c]{@{}c@{}}$2nr$\\ $+ mr$\end{tabular} & \begin{tabular}[c]{@{}c@{}}$2nr$\\ $+ 1$\end{tabular}\\ 
\midrule
    Full-Rank Gradients & \ding{52} & \ding{52} & \ding{52} & \ding{56} & \ding{56} \\
    Full-Rank Weights & \ding{52} & \ding{52} & \ding{52} & \ding{52} & \ding{52}  \\
    Pre-Training & \ding{52} & \ding{52} & \ding{52} & \ding{52} & \ding{56} \\
    Fine-Tuning & \ding{52} & \ding{52} & \ding{52} & \ding{52} & \ding{52} \\      
    \textit{w.o.} SVD & \ding{52} & \ding{52} & \ding{56} & \ding{56} & \ding{52} \\     
\bottomrule
\end{tabular}
}
\end{center}
\end{table}

In terms of memory efficiency, random projection allows \name to eliminate the need to store a projection matrix, achieving significant memory efficiency. Additionally, \name is robust to rank reduction; halving the rank has minimal impact on pre-training performance (see Table~\ref{tab:pre-train}). Our extreme variant, \namec, further maximizes memory efficiency by reducing optimizer state costs to a const number $2n+2$, making it comparable to the memory footprint of SGD, yet it retains or even surpass AdamW-level performance, as confirmed in the following experiments.

\begin{table*}[t]
    \centering
    \caption{Comparison of pretraining perplexity across various memory-efficient training approaches. We pretrain the LLaMA models with model size ranging from 60M to 1B on the C4~\cite{raffel2020exploring} dataset and report the validation perplexity. The memory overhead focus solely on weights and optimization states. Results marked with $^\text{\ding{171}}$ are collected from~\cite{zhao2024galore}. By default, we set the rank to one-quarter of the original dimension for all low-rank-based training approaches, while results marked with $^\text{\ding{61}}$ indicate the use of a halved rank, \textit{i.e.}, one-eighth of the original dimension.
    \textcolor{mydarkblue}{\textbf{For a fair comparison, we keep the same training settings as GaLore/Fira, which is not tuned on the \name series. We can further tune the hyperparameters, e.g., the learning rates, for optimal performance. Here, we report the \namec with a learning rate of 0.02 (not 0.01 in GaLore), achieving stronger results, marked with $^\ddagger$.
    }}
    }
    \label{tab:pre-train}
    \resizebox{0.9 \linewidth}{!}{
    \begin{tabular}{c|cc|cc|cc|cc}
    \toprule
    \multirow{2}{*}{Methods} & \multicolumn{2}{c|}{60M} &  \multicolumn{2}{c|}{130M} &  \multicolumn{2}{c|}{350M} &  \multicolumn{2}{c}{1B} \\
    &  \small Perplexity &  \small Memory &  \small Perplexity &  \small Memory &  \small Perplexity &  \small Memory &  \small Perplexity &  \small Memory \\ \midrule
    AdamW$^\text{\ding{171}}$ & \small 34.06 & \small 0.36G & \small 25.08 & \small 0.76G & \small 18.80 & \small 2.06G & \small 15.56 & \small 7.80G \\ \midrule
    Low-Rank$^\text{\ding{171}}$ & \small 78.18 & \small 0.26G & \small 45.51 & \small 0.54G & \small 37.41 & \small 1.08G & \small 142.53 & \small 3.57G \\
    LoRA$^\text{\ding{171}}$ & \small 34.99 & \small 0.36G & \small 33.92 & \small 0.80G & \small 25.58 & \small 1.76G & \small 19.21 & \small 6.17G \\
    ReLoRA$^\text{\ding{171}}$ & \small 37.04 & \small 0.36G & \small 29.37 & \small 0.80G & \small 29.08 & \small 1.76G & \small 18.33 & \small 6.17G \\
    GaLore$^\text{\ding{171}}$ & \small 34.88 & \small 0.24G & \small 25.36 & \small 0.52G & \small 18.95 & \small 1.22G & \small 15.64 & \small 4.38G \\ 
    Fira & \small \textbf{31.06} & \small 0.24G & \small \textbf{22.73} & \small 0.52G & \small 17.03 & \small 1.22G & \small 14.31 & \small 4.38G \\\midrule 
    \rowcolor{gray4}
    \name \textit{w.} \texttt{SVD} & \small 31.26 & \small 0.24G & \small 22.84 & \small 0.52G & \small \textbf{16.67} & \small 1.22G & \small \bf 14.10 & \small 4.38G  \\
    \rowcolor{gray4}
    \name & \small 31.55 & \small 0.24G & \small 22.94 & \small 0.52G & \small \textbf{16.85} & \small 1.22G & \small \bf 14.20 & \small 4.38G  \\
    \rowcolor{gray3}
    \name$^\text{\ding{61}}$ & \small 31.26 & \small 0.18G & \small 23.18 & \small 0.39G & \small \textbf{16.98} & \small 0.95G & \small \bf 14.25 & \small 3.49G  \\
    \rowcolor{gray2}
    \namec  & \small 31.93 & \small 0.12G & \small 23.53 & \small 0.25G & \small 17.18 & \small 0.69G & \small \bf 14.17 & \small 2.60G  \\
    \rowcolor{teal!10}
    \namec$^\ddagger$  & \small 30.95 & \small 0.12G & \small 22.85 & \small 0.25G & \small 16.63 & \small 0.69G & \small 13.95 & \small 2.60G  \\
    \bottomrule
    \end{tabular}
}
\end{table*}

\section{Experiments}
In Section~\ref{sec:train} and~\ref{sec:ft}, we systematically evaluate \name on various pre-training and downstream tasks, respectively. Section~\ref{sec:memory} compares the memory overhead and throughput of different approaches, while Section~\ref{sec:ablation} presents extensive ablation studies that analyze the impact of low-rank projection methods, rank quantity, scaling factor granularity, and provide a detailed comparison of training curves.
Section~\ref{sec:insight} shares preliminary insights on why a stateless \name
can surpass AdamW in certain scenarios,

\subsection{Memory-Efficient Pre-training with \name}~\label{sec:train}

\vspace{-4mm}
We demonstrate that \name achieves \textbf{superior pre-training performance} across various sizes of LLaMA models (60M to 7B) on the C4 dataset~\cite{raffel2020exploring}, with up to a $2.80$ reduction in validation perplexity. Additionally, \namec 
uses negligible memory budget for optimization states while outperforming both AdamW~\cite{loshchilov2017decoupled} and GaLore~\cite{zhao2024galore}.

\vspace{-4mm}
\paragraph{Setup.} We consider the LLaMA series models for pre-training, with sizes ranging from 60M to 7B. Following the training configurations used in prior works~\cite{zhao2024galore, lialin2023relora}, we pre-train each model from scratch, with a detailed description in Appendix.\ref{sub:setting}. 
The C4 dataset~\cite{raffel2020exploring}, a comprehensive corpus derived from Common Crawl data and meticulously filtered and cleaned, is used for pre-training. All experiments are conducted in BF16 data format without other quantization. 

\vspace{-4mm}
\paragraph{Baselines.} \label{sec:baselines}
For comparative analysis, we include the following baselines in our evaluation:
(i) \texttt{AdamW}: We pre-train the models using the original AdamW optimizer~\cite{loshchilovdecoupled}, maintaining the weights, gradients, and optimizer states in their full-rank format. (ii) \texttt{Low-Rank}: This approach decomposes the model weights into two low-rank matrices ($W = UV$), with both $U$ and $V$ optimized using AdamW. (iii) \texttt{LoRA}: LoRA~\cite{hu2021lora} employs low-rank adapters for memory-efficient training by decomposing the original weights as $W = W_0 + UV$. During training, only $U$ and $V$ are optimized with AdamW, while the backbone weights $W_0$ remain frozen.
(iv) \texttt{ReLoRA}: ReLoRA~\cite{lialin2023relora} is an enhanced version of LoRA specifically designed for pre-training, where the low-rank adapters $UV$ are periodically merged back into the original weights $W$ during training. 
(v) \texttt{GaLore}: GaLore~\cite{zhao2024galore} projects gradients, rather than weights, into a low-rank space, effectively reducing memory consumption for optimizer states. 
(vi) \texttt{Fira}: Fira~\cite{chen2024fira} further improves GaLore by incorporating the error residual of low-rank gradient into training. 

\vspace{-2mm}
\paragraph{Main Results.} We evaluate \name and its two variants: \name \textit{w.} \texttt{SVD}, which replaces the original random projection with SVD; and \namec, which uses a rank of 1 and computes the scaling factor in a tensor-wise manner. The memory cost of optimization states in \namec is negligible, as the rank used is substantially smaller than the original dimension, \textit{e.g.}, 1 \textit{vs.} 512 for LLaMA-60M and  1 \textit{vs.} 2048 for LLaMA-1B. Results are reported in Table~\ref{tab:pre-train}, from which several observations can be made: (i) \textbf{Performance under the same memory budget}: When the rank is set to one-quarter of the original dimension, \name consistently outperforms GaLore, achieving up to a $3.33$ reduction in perplexity; (ii) \textbf{Comparison with full-rank AdamW}: \name demonstrates superior performance while using significantly less memory. Notably, \namec incurs a memory cost similar to that of SGD while significantly outperforming AdamW and vanilla SGD is known to fail on training transformer-based models~\cite{zhang2024transformers}; (iii) \textbf{Robustness across projection methods and rank sizes}: \name exhibits robust performance under various subspace projection methods and rank sizes. For instance, with LLaMA-350M, enabling SVD projection results in only a 0.18 improvement of perplexity, while SVD is known for its time-consuming nature~\cite{zhang2024q}. This indicates that \name can maintain efficiency even without SVD, improving end-to-end throughput. Additionally, halving the rank has negligible impact on \name’s performance, further demonstrating its effectiveness across different rank budgets. Section~\ref{sec:ablation} provide a more thorough analysis of the impact of rank quantity and projection methods; (iv) \textbf{Comparison with Fira}: \name is more favorable when dealing with larger models and more training tokens. For smaller models like 60M and 130M, Fira shows slightly better performance, but \name consistently surpasses Fira as model size and training tokens increase. An in-depth comparison of training performance across model sizes and training tokens is provided in Section~\ref{sec:ablation}, \textit{A4}. The above results validate the effectiveness of \name on pre-training tasks, demonstrating that it achieves superior performance while requiring negligible memory costs for optimization states compared to AdamW.

\begin{table}[t]
    \caption{\small{Pre-training LLaMA 7B on C4 dataset for 150K steps. Validation perplexity and memory estimate (optimization states only) are reported. Results marked with $^\text{\ding{171}}$ are collected from~\cite{zhao2024galore}.} \name uses the rank of 256, and \namec uses the rank of 1.}
    \label{tab:7b_eval}
    \centering
    \resizebox{0.98 \linewidth}{!}{
    \begin{tabular}{l|c|cccc}
    \toprule
    & \multicolumn{1}{c|}{\begin{tabular}[c]{@{}c@{}}\textbf{Optimizer}\\ \textbf{Memory}\end{tabular}}       & \textbf{40K} & \textbf{80K} & \textbf{120K} & \textbf{150K} \\
    \midrule
    8-bit Adam$^\text{\ding{171}}$ & 13G & 18.09 & 15.47 & 14.83 & 14.61 \\
    8-bit \lowrank{}$^\text{\ding{171}}$ & 4.9G & 17.94 & 15.39 & 14.95 & 14.65 \\
    \rowcolor{gray4}
    \name& \textbf{1.6G} & 17.55 & 14.39 & 13.23 & \textbf{13.02 }\\
    \rowcolor{gray2}
    \namec & \textbf{0.0G} & 18.03& 14.60 & 13.32 & \textbf{13.09} \\
    \midrule
    Tokens (B) & & 5.2 & 10.5 & 15.7 & 19.7 \\
    \bottomrule
    \end{tabular}
}
\end{table}

\begin{table*}[htb]
\centering
\caption{Zero-shot performance of LLaMA-350M models pretrained with AdamW and \name-series on commonsense and math reasoning tasks.}
\resizebox{0.95\textwidth}{!}{
\begin{tabular}{c|c|c|c|c|c|c|c|c|c|c|c|c|c}
    \toprule
    \multicolumn{14}{c}{\textbf{Sequence Length: 256}} \\
    \midrule
    \small Method & \small Memory & \small Perplexity & \small BoolQ & \small RTE & \small HS & \small WG & \small OBQA & \small ARC-E & \small ARC-C & \small PIQA & \small SciQ & \small MathQA & \small Average \\
    \midrule
    \small AdamW & \small 1.37G & \small 18.80 & \small 0.5881 & \small 0.4729 & \small 0.3286 & \small 0.5335 & \small 0.304 & \small 0.3615 & \small 0.2167 & \small 0.6387 & \small 0.591 & \small 0.2047 & \small 0.3554 \\
    \rowcolor{gray4}
    \small \name & \small 0.34G & \small 16.85 & \small 0.5165 & \small 0.4729 & \small 0.3528 & \small 0.5146 & \small 0.318 & \small 0.3792 & \small 0.2517 & \small 0.6632 & \small 0.592 & \small 0.2188 & \small \textbf{0.3681} \\
    \rowcolor{gray2}
    \small \namec & \small 0.00G & \small 17.18 & \small 0.5434 & \small 0.4729 & \small 0.3481 & \small 0.5162 & \small 0.320 & \small 0.3653 & \small 0.2474 & \small 0.6469 & \small 0.591 & \small 0.2178 & \small \textbf{0.3654} \\
    \midrule
    \multicolumn{14}{c}{\textbf{Sequence Length: 1024}} \\
    \midrule
    \small Method & \small Memory & \small Perplexity & \small BoolQ & \small RTE & \small HS & \small WG & \small OBQA & \small ARC-E & \small ARC-C & \small PIQA & \small SciQ & \small MathQA & \small Average \\
    \midrule
    \small AdamW & \small 1.37G & \small 16.30 & \small 0.4917 & \small 0.4693 & \small 0.3688 & \small 0.5233 & \small 0.332 & \small 0.3729 & \small 0.2449 & \small 0.6534 & \small 0.609 & \small 0.2064 & \small 0.3712 \\
    \rowcolor{gray4}
    \small \name & \small 0.34G & \small 15.64 & \small 0.5373 & \small 0.4693 & \small 0.3850 & \small 0.4925 & \small 0.322 & \small 0.3788 & \small 0.2483 & \small 0.6681 & \small 0.624 & \small 0.2127 & \small \textbf{0.3840} \\
    \rowcolor{gray2}
    \small \namec & \small 0.00G & \small 16.12 & \small 0.5376 & \small 0.4693 & \small 0.3707 & \small 0.5217 & \small 0.324 & \small 0.3758 & \small 0.2312 & \small 0.6638 & \small 0.619 & \small 0.2224 & \small \textbf{0.3785} \\
    \bottomrule
\end{tabular} 
 }
\label{tab:downstream}
\vspace{-10pt}
\end{table*}

\vspace{-2mm}
\paragraph{Scaling up to Pre-training LLaMA-7B.}
We evaluate the pre-training of a LLaMA-7B model using AdamW, GaLore, and our \name series (\name ($r=256$) and \namec ($r=1$)) on an 8× A100 80GB setup. To ensure performance consistency, a total batch size of 512 per epoch is maintained, while micro-batch sizes are adjusted based on the memory consumption of each method. AdamW is limited to a micro-batch size of 4 due to memory constraints, whereas GaLore matches the memory usage of our methods with a micro-batch size of 8.
Considering the extended training duration typically required for fully training a LLaMA-7B model with AdamW, we allocate a fixed training time budget of around two weeks (\textbf{15 Days}) for a fair and practical comparison. The training curve, showing validation perplexity recorded every 1000 steps, is in Fig.~\ref{fig:7b-curve}.

Our experiments reveal two key benefits of the \name series on the large 7B model:
(i) \textbf{Accelerated training through saved memory, enabling larger batch sizes.} The significant memory efficiency of the \name series allows for larger batch sizes, resulting in $\sim$3× faster training throughput compared to AdamW and $\sim$2× faster throughput compared to GaLore. Notably, our methods are the only ones to complete the pre-training within the half a month timeframe.
(ii) \textbf{Superior model performance with best perplexity and reduced optimizer overhead.} Despite its efficiency, \name delivers better perplexity results than GaLore, even when GaLore employs a high rank of 1024. Midway through training, \name surpasses GaLore in performance, marking a critical intersection point where our approach demonstrates clear superiority. Furthermore, as shown in Tab.~\ref{tab:7b_eval}, the \name series achieves $>$1.5 perplexity improvement compared to the 8-bit versions of Adam and GaLore, all while maintaining significantly lower memory usage for optimizer states.

\vspace{-2mm}
\paragraph{Downstream Task Performance.}
While perplexity may not precisely assess model quality~\cite{jaiswal2023compressing}, we further evaluate the pretrained models on a suite of zero-shot common-sense and math reasoning tasks.  
We use LLaMA-350M trained with sequence lengths of 256 and 1024, selecting the best AdamW checkpoint by sweeping its learning rates.  
Table~\ref{tab:downstream} presents results across multiple datasets, including BoolQ~\cite{clark2019boolq}, RTE~\cite{wang2018glue}, HellaSwag (HS)~\cite{zellers2019hellaswag}, Winogrande (WG)~\cite{sakaguchi2021winogrande}, OpenBookQA (OBQA)~\cite{mihaylov2018can}, ARC (ARC-Easy (ARC-E), ARC-Challenge (ARC-C))~\cite{clark2018think}, PIQA~\cite{bisk2020piqa}, SciQ~\cite{SciQ}, and MathQA~\cite{amini2019mathqa}.  
We confirm that models pretrained with the \name series not only achieve \textbf{lower perplexity} but also demonstrate \textbf{superior performance on downstream tasks} compared to AdamW.

\begin{center} 
\conclusionbox{ 
\textit{Conclusion \ding{172}: The \name series \textbf{optimize memory usage to a SGD-like level}, \textbf{enhances model quality over AdamW}, and \textbf{accelerates pre-training} by enabling larger batch sizes. 
This makes \name a highly practical and efficient solution for large-scale LLM pre-training.} 
} 
\end{center}

\subsection{Memory-Efficient Fine-tuning with \name}~\label{sec:ft}

\vspace{-4mm}
Pre-training large foundation models typically demands thousands of GPUs and months of training, making it feasible only for large organizations. 
As a result, fine-tuning these models has become a more practical approach among engineers and researchers. Here, we thoroughly evaluate the performance of \name in fine-tuning scenarios.

\vspace{-2mm}
\paragraph{Setup.} We employ three open-source pre-trained models in the fine-tuning experiments, including LLaMA-3.2-1B, LLaMA-3-8B~\cite{llama3modelcard}, Gemma-7B and Mistral-7B~\cite{jiang2023mistral}. The downstream tasks are divided into two categories: (i) Eight common-sense reasoning tasks: Winogrande, PIQA, SIQA, OpenBookQA, HellaSwag, BoolQ, and ARC-Easy and ARC-Challenge; (ii) MMLU~\cite{hendrycks2020measuring} tasks across various domains: STEM, Social Sciences, Humanities and others.

\vspace{-2mm}
\paragraph{Baseline.} We compare \name against several baselines used during the pre-training experiments, including Full-rank AdamW~\cite{loshchilov2017decoupled}, LoRA~\cite{hu2021lora}, GaLore~\cite{zhao2024galore}, and Fira~\cite{chen2024fira}. Details of these baselines are summarized in Section~\ref{sec:baselines}. Additionally, we include DoRA~\cite{liu2024dora}, an effective fine-tuning approach.
We set the rank to 32 and 8 for common-sense reasoning and MMLU tasks, respectively.
\namec uses a rank of 1.

\begin{table*}[htb]
\centering
\caption{Comparison of various finetuning approaches on common-sense reasoning tasks. Experiments are conducted with Llama-3.2-1B based on the implementation from \cite{liu2024dora}.}
\resizebox{0.8 \textwidth}{!}{\begin{tabular}{c|cccccccc|c}
\toprule
\small Methods  &\small  WG  & \small PIQA  & \small SIQA  & \small OBQA  &  \small HS   & \small BoolQ   & \small ARC-E   & \small ARC-C   & \small Average \\
\toprule
\small AdamW & \small 68.19 & \small 76.12  & \small 72.36  & \small 69.00 & \small  69.19  & \small  64.34 & \small 72.22  & \small 55.12 & \small 68.07 \\  \midrule 

\small LoRA & \small  67.56 & \small 63.28 & \small 71.65  & \small 68.20 & \small  19.13 & \small 63.58 & \small 67.30  & \small 52.99 & \small 59.21 \\
\small DoRA & \small 68.98 & \small 74.70 & \small 72.47 & \small 64.80  & \small 63.93 & \small 64.01 & \small 69.32 & \small 52.82 & \small 66.38 \\
\small GaLore & \small 62.75   & \small 72.63 & \small 68.17  & \small 62.20 & \small 47.81 & \small 58.99 & \small  68.94 & \small 47.61 & \small 61.14 \\
\small Fira & \small  71.82  & \small 77.20 & \small  73.08  & \small 69.00 & \small 68.21  & \small 64.31 & \small  73.40 & \small  54.78 & \small 68.98 \\ \midrule
\rowcolor{gray4}
\small{\name \textit{w.} \texttt{SVD}} & \small 70.88 & \small  77.69  & \small 72.52  & \small  70.60 & \small 68.19  & \small 63.00 & \small  74.03  & \small  55.72 & \small  \bf 69.08 \\ 
\rowcolor{gray4}
\small \name & \small 70.40  & \small 76.93 & \small 72.72 & \small  70.60 & \small 63.75  & \small 62.69 & \small 73.40  & \small 55.20 & \small   \bf 68.21 \\ 
\rowcolor{gray2}
\small \namec & \small 67.64 & \small 76.50 & \small 72.88 & \small 69.60 & \small 66.54 & \small 64.22 & \small 72.98  & \small 55.46 & \small   \bf 68.23 \\ 

\bottomrule 
\end{tabular}}
\label{tab:commonse_sense}
\vspace{-15pt}
\end{table*}

\begin{table}[htb]
    \caption{Comparison results of various memory-efficient fine-tuning algorithms on MMLU tasks. For Galore, Fira, \name, and \namec, we report the best accuracy obtained by sweeping the learning rate within the range [5e-6, 7.5e-6, 1e-5, 2.5e-5, 5e-5, 7.5e-5, 1e-4,1.5e-4, 2e-4].} 
    \label{tab:fine_tuning_mmlu}
    \centering 
    \resizebox{1\linewidth}{!}{
    \begin{tabular}{c|c|cccc|c} \toprule
    \small Model & \small Methods & \small STEM & \multicolumn{1}{c}{\begin{tabular}[c]{@{}c@{}}Social\\ Sciences\end{tabular}}
 & \small Humanities & \small Other & \small Average \\ \midrule
    \multirow{6}{*}{\small LLaMA-3-8B} 
    & \small Full  & \small 54.27 & \small 75.66 & \small 59.08 & \small 72.80 & \small 64.85  \\
    & \small LoRA & \small 53.00 & \small 74.85 & \small 58.97 & \small 72.34 & \small 64.25 \\
    & \small GaLore & \small 54.50  & \small 75.11 & \small 58.59 & \small 72.03 & \small 64.43 \\
    & \small Fira  & \small 53.53  & \small 75.46 & \small 58.59 & \small 72.09 & \small 64.32  \\\cmidrule{2-7}
    \rowcolor{gray4}
    & \small{\name \textit{w.} \texttt{SVD}} & \small 54.73  & \small 75.46 & \small 58.72 & \small 72.68 & \small \textbf{64.7}6  \\ 
    \rowcolor{gray4}
    & \small \name &  \small 54.37  & \small 75.86 & \small 58.18 & \small 71.69 & \small 64.35  \\
    \rowcolor{gray2}
    & \small \namec & \small 54.40 & \small 75.37 & \small 58.72 & \small 71.59 & \small 64.41 \\ 
    \midrule

    \multirow{6}{*}{\small Gemma-7B}  
    & \small Full  & \small 30.03  & \small 37.16  & \small 34.08  & \small 35.47 & \small 34.21  \\
    & \small LoRA   & \small 26.23 & \small 34.94 & \small 30.88  & \small 36.96  & \small 32.18 \\
    & \small GaLore  & \small 25.47  & \small 33.21 & \small 31.07 & \small 33.71 & \small 30.95 \\
    & \small Fira   & \small 29.03  & \small 35.27 & \small 32.40 & \small 36.52 & \small 33.26 \\\cmidrule{2-7}

    \rowcolor{gray4}
    & \small{\name \textit{w.} \texttt{SVD}}    & \small 29.20  & \small 40.42 & \small 32.40 & \small 38.94 & \small \textbf{34.98}  \\ 
    \rowcolor{gray4}
    & \small \name & \small 27.53  & \small 36.97 & \small 33.99 & \small 36.40 & \small 33.81 \\
    \rowcolor{gray2}
    & \small \namec  & \small 27.30  & \small 33.83 & \small 31.61 & \small 33.77 & \small 31.67 \\ 
    \midrule

    \multirow{6}{*}{\small Mistral-7B} 
    & \small Full & \small 52.40  & \small 72.95  & \small 55.16  & \small 69.05  & \small 61.67  \\
    & \small LoRA & \small 52.13 & \small 72.46 & \small 55.05  & \small 68.77 & \small 61.41 \\
    & \small GaLore & \small 51.87  & \small 72.82 & \small 54.94 & \small 69.49 & \small 61.56 \\
    & \small Fira & \small 52.80  & \small 72.85 & \small 55.07 & \small 69.11 & \small 61.72 \\\cmidrule{2-7}
    \rowcolor{gray4}
    & \small{\name \textit{w.} \texttt{SVD}}     & \small 52.43  & \small 73.28 & \small 55.05 & \small 69.24 & \small \textbf{61.76} \\
    \rowcolor{gray4}
    & \small \name & \small 51.63  & \small 73.12 & \small 54.90 & \small 69.58 & \small 61.58 \\
    \rowcolor{gray2}
    & \small \namec  & \small 51.97 & \small 72.89 & \small 54.43 & \small 69.18 & \small 61.35 \\ 
    \bottomrule
    \end{tabular}
}
\vspace{-15pt}
\end{table}

\vspace{-2mm}
\paragraph{Main Results.} As shown in Table~\ref{tab:commonse_sense} and Table~\ref{tab:fine_tuning_mmlu}, \name consistently matches or outperforms other baselines, achieving up to an $1.01$ average accuracy improvement over full-rank AdamW on common-sense reasoning tasks. Notably, compared to AdamW, \name requires only a rank-32 space for optimization states, while \namec uses a rank of 1, resulting in negligible memory costs for optimization states. Despite this, both approaches deliver a clear margin of accuracy improvement over AdamW on commonsense tasks while maintaining comparable performance on the MMLU tasks.

\begin{center} 
\conclusionbox{ 
\textit{Conclusion \ding{173}: The \name series establishes itself as a compelling \textbf{memory-efficient full-parameter} fine-tuning method, delivering \textbf{on-par or better performance compared to AdamW}.
}
}
\end{center}

\begin{table}[htb]
    \centering
    \caption{Optimizer step time (in seconds) across LLaMA-1B and LLaMA-7B with a sequence length of 1024 on a single A100 GPU. Results are averaged over 400 steps, with 100 warm-up steps and low-rank projection matrices updated every 200 steps. Batch sizes are set to 16 (1B) and 4 (7B), the maximum batch size supported by AdamW. Lower values indicate less overhead.}
    \label{tab:opt_step_time}
    \resizebox{0.98\linewidth}{!}{
    \begin{tabular}{c|c|c|c|c|c|c}
        \toprule
        \multirow{3}{*}{\small \begin{tabular}[c]{@{}c@{}}Model\\ Size\end{tabular}} & \multirow{3}{*}{\small \begin{tabular}[c]{@{}c@{}} Bwd\\ Time(s)\end{tabular}} & \multicolumn{5}{c}{\small Optimizer Step Time (s)} \\
        \cmidrule(lr){3-7}
        & & \small AdamW & \small \name & \small \begin{tabular}[c]{@{}c@{}} \name\\ \texttt{-Mini}\end{tabular} & \small GaLore & \small Fira \\
        \midrule
        1B  & 1.069 & \textbf{0.036} & 0.051 & 0.048 & 0.371 & 0.421 \\
        7B  & 0.712 & 0.173 & 0.159 & \textbf{0.142} & 2.874 & 3.086 \\
        \bottomrule
    \end{tabular}}
    \vspace{-15pt}
\end{table}
\begin{table*}[t]
    \centering
    \caption{Validation of pretraining perplexity of \name-series combined with int-8 weight quantization strategy~\cite{zhang2024q}. \name-series uses a quantization group size of 128.
    $^\text{\ding{171}}$ are collected from~\cite{zhao2024galore} and ~\cite{zhang2024q}.}
    \label{tab:q-pre-train}
    \resizebox{0.75 \linewidth}{!}{
    \begin{tabular}{c|cc|cc|cc}
    \toprule
    \rowcolor{white}
    \multirow{2}{*}{Methods} & \multicolumn{2}{c|}{60M} &  \multicolumn{2}{c|}{130M} &  \multicolumn{2}{c}{350M} \\
    &  \small Perplexity &  \small Memory &  \small Perplexity &  \small Memory &  \small Perplexity &  \small Memory \\ \midrule
    AdamW$^\text{\ding{171}}$ & \small 34.06 & \small 0.36G & \small 25.08 & \small 0.76G & \small 18.80 & \small 2.06G \\ \midrule
    GaLore$^\text{\ding{171}}$ & \small 34.88 & \small 0.24G & \small 25.36 & \small 0.52G & \small 18.95 & \small 1.22G \\ 
    Q-GaLore$^\text{\ding{171}}$ & \small 34.88 & \small 0.18G & \small 25.53 & \small 0.39G & \small 19.79 & \small 0.88G  \\ 
    \rowcolor{gray4}
    \name & \small 31.55 & \small 0.24G & \small 22.94 & \small 0.52G & \small 16.85 & \small 1.22G \\
    \rowcolor{gray4}
    Q-\name & \small 31.97 & \small 0.18G & \small 24.16 & \small 0.39G & \small 18.79 & \small 0.88G \\
    \rowcolor{gray2}
    \namec  & \small 31.93 & \small 0.12G & \small 23.84 & \small 0.25G & \small 17.18 & \small 0.69G  \\
    \rowcolor{gray2}
    Q-\namec  & \small 33.05 & \small 0.06G & \small 24.70 & \small 0.12G & \small 18.90 & \small 0.35G  \\
    \bottomrule
    \end{tabular}
}
\vspace{-10pt}
\end{table*}

\subsection{End-to-End System-level Benefits} ~\label{sec:memory}

\vspace{-4mm}
We further validate the system-level benefits, end-to-end throughput, and memory overhead by running LLaMA-7B on 8$\times$ A100-80GB GPUs, comparing \name and \namec against AdamW in Fig.~\ref{fig:teasor}.

\vspace{-2mm}
\textbf{Higher throughput:} \name achieves significantly higher throughput than AdamW, particularly under limited hardware resources, as AdamW struggles with large batch sizes due to memory constraints. With \name, we can scale the batch size up to 4$\times$ that of AdamW, resulting in up to a 3$\times$ improvement in throughput. Additionally, the \name series avoids the extra computational overhead associated with SVD that are known to be expensive.

\vspace{-2mm}
\textbf{Much lower memory usage:} With a batch size of 4, AdamW already reaches the memory limit with a memory usage of $\sim$79 GB, while \name and \namec require only $\sim$70 GB and $\sim$68 GB, respectively, for a batch size of 16.
This study highlights that using AdamW not only results in high memory costs but also reduces training efficiency, as models become memory-bound, preventing full utilization of computational resources. By using \name, we can effectively increase batch size, accelerating large-scale training with even better performance.

\vspace{-4mm}
\paragraph{Negligible Optimizer Step Overhead.}
Beyond end-to-end throughput, we evaluate optimizer step time to demonstrate the minimal overhead of the \name series, which relies only on \textit{lightweight random projections}, in Table ~\ref{tab:opt_step_time}. 
We benchmark LLaMA-1B and 7B with a sequence length of 1024, using batch sizes of 16 and 4, respectively—the largest batch sizes that AdamW supports.

Unlike GaLore and Fira, which suffer from expensive SVD updates (taking approximately 10 minutes for the 7B model),  
\textbf{the \name series incurs minimal overhead} by relying only on cheap random projections.  
While \name introduces an additional step to compute the scaling factor and applies NL,  
\textit{projecting gradients into a low-rank space reduces the overhead of maintaining first-order and second-order moments.}
Notably, on the 7B model, the \name series achieves even faster optimizer step times than AdamW,  
demonstrating its efficiency in both compute and memory usage.
Overall, the optimizer step overhead remains negligible compared to the backward pass, particularly for larger batch sizes, ensuring that \name enables superior memory efficiency while being a lightweight solution.

\vspace{-2mm}
\paragraph{\textbf{\namec unlocks pre-training LLaMA-13B on A100 80GB without system-Level optimization:}}
Leveraging the exceptional memory efficiency of \namec, we are the first to enable the pre-training of a LLaMA-13B model A100 80GB GPU with naive DDP, without requiring other system-level optimizations, such as model sharding.
This breakthrough not only simplifies deployment by reducing engineering complexity but also empowers researchers to scale up model sizes effortlessly with \namec.

\vspace{-2mm}
\paragraph{\textbf{Combination with weight quantization unlocks pre-training LLaMA-7B under 12 GB:}}  
With our methods significantly reducing optimizer memory costs, the memory consumed by model weights becomes the next dominant factor. To further address this challenge, we integrate our approach with the Int8 weight quantization method proposed in Q-GaLore~\cite{zhang2024q}, enabling even greater memory savings. The results, detailed in Table~\ref{tab:q-pre-train}, demonstrate that our Q-\name series effectively minimizes memory usage for both optimizer and weight components while maintaining on-par or better pre-training perplexity compared to full-precision AdamW. Moreover, Q-\name achieves a clear performance margin over Q-GaLore, underscoring its superiority in both memory efficiency and model quality.

By successfully combining \name-series with quantization, we enable—\textbf{for the first time}—the pre-training of a LLaMA-7B model using just 12 GB of memory (Q-\namec), assuming a layer-wise gradient update strategy~\cite{lv2023full} is employed. This represents a significant breakthrough in making LLM pre-training feasible on low-end GPUs, eliminating the dependency on high-end hardware traditionally required for LLM training. 
Our approach democratizes access to LLM pre-training, making it accessible to a broader range of researchers and organizations with limited resources.

\begin{center} 
\conclusionbox{ 
\textit{Conclusion \ding{174}: 
The \name series \textbf{significantly reduces optimizer memory usage} with \textbf{minimal compute overhead}, enabling \textbf{higher throughput}, \textbf{improved model scalability}, and \textbf{more friendly low-end GPU training} for LLMs.  
\name takes a pivotal step toward \textbf{democratizing LLM training}, making large-scale model training more accessible and efficient.
}
} 
\end{center}

\begin{figure*}[!htb]
    \centering
    \includegraphics[width=1 \linewidth]{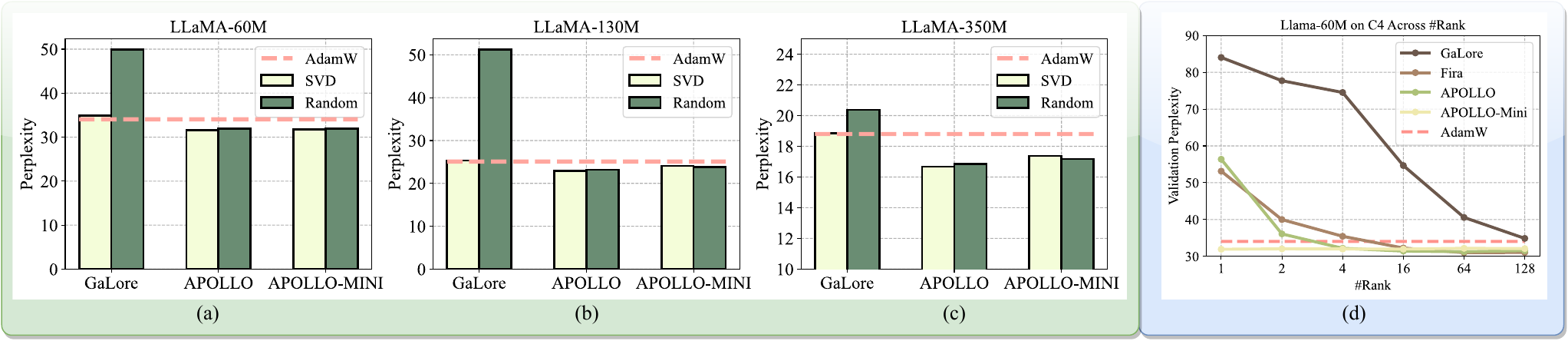}
    \vspace{-4mm}
    \caption{(a-c) Comparison results of various optimization methods using singular value decomposition or random projection. The experiments were conducted on LLaMA-60M/130M/350M models for C4 pretraining tasks. (d) Validation perplexity with varying rank sizes, where 128 is one-quarter of the original model dimension. The red dashed line indicates the performance of full-rank AdamW.}
    \label{fig:svd_rp}
    \vspace{-10pt}
\end{figure*}

\subsection{Extra Investigation and Ablation Study}~\label{sec:ablation}

\vspace{-2mm}
This section presents multiple experimental investigations, focusing on four key research questions: \textit{Q1:} How can the scaling factor subspace be identified? \textit{Q2:} Is \name sensitive to the number of rank? \textit{Q3:} What is an appropriate granularity for scaling factors? \textit{Q4:} How do detailed comparisons evolve throughout the training process?
\textit{Q5:} How \name performs in long-context training setting?

\vspace{-2mm}
\paragraph{\textit{A1}: Similar performance between Random Projection (RP) and Singular Value Decomposition (SVD).} Previous low-rank gradient-based approaches~\cite{zhao2024galore} rely on SVD to identify the gradient subspace, frequently updated during training. This process is time-consuming, thereby affecting training throughput. For a LLaMA-7B model, each SVD operation takes approximately ten minutes, resulting in an additional 25 hours of training time over 150K steps when the subspace is updated every 1,000 steps. 
To alleviate this overhead, \cite{zhang2024q} employs a lazy subspace updating strategy, though it still incurs substantial SVD costs. 
In this section, we demonstrate that \name performs effectively with random projection, significantly reducing the heavy SVD costs present in previous memory-efficient training algorithms. We pre-train LLaMA models ranging from 60M to 350M on the C4 dataset using GaLore, \name, and \namec, reporting results for both SVD and random projection in each method. As shown in Figure~\ref{fig:svd_rp} (a-c), GaLore is significantly impacted by random projection, failing to match the performance of AdamW (red dashed line). In contrast, both \name and \namec demonstrate strong robustness with random projection, even slightly outperforming SVD in certain cases, such as \namec on LLaMA-350M. These results confirm the effectiveness of \name under random projection, thereby addressing the throughput challenges present in previous low-rank gradient methods.

\vspace{-2mm}
\paragraph{\textit{A2}: \namec remains effective even with a rank of 1.} We carry out an ablation study on pre-training LLaMA-60M with different rank budgets, as shown in Figure~\ref{fig:svd_rp} (d). The results demonstrate that GaLore’s performance degrades significantly as the rank decreases, matching full-rank AdamW only when the rank is set to 128 (one-quarter of the original dimension), which limits its effectiveness in extreme low-rank scenarios. In contrast, \name exhibits much better robustness with smaller rank settings compared to both GaLore and Fira, achieving performance comparable to full-rank AdamW even with lower ranks.

Interestingly, \namec shows the best rank efficiency, remaining effective even with a rank of 1, clearly outperforming AdamW. 
By averaging the gradient scaling factor across different channels, \namec seems to effectively mitigate the errors introduced by low-rank projection. This capability allows \namec to achieve SGD level memory cost while reaching superior performance than AdamW.

\begin{table}[!htb]
    \centering
    \caption{Ablation study on the granularity of gradient scaling factors. Perplexity on the validation set is reported.}
    \label{tab:granularity}
    \resizebox{0.9\linewidth}{!}{
    \begin{tabular}{cc|ccc}
    \toprule
    \rowcolor{white}
    \small Methods & \small Granularity & \small 60M & \small 130M & \small 350M  \\ \midrule
    \multicolumn{2}{c|}{AdamW} & \small 34.06 & \small 25.08 & \small 18.80  \\ 
    \multicolumn{2}{c|}{GaLore} & \small 34.88 & \small 25.36 & \small 18.95  \\ \midrule
    \multirow{2}{*}{\name \textit{w.} \texttt{SVD}} & \small \textit{Channel} & \small 31.26 & \small 22.84 & \small 16.67  \\ 
    & \small \textit{Tensor} & \small 31.77 & \small 23.86 & \small 16.90 \\ \midrule
    \multirow{2}{*}{\name} & \small \textit{Channel} & \small 31.55 & \small 22.94 & \small 16.85  \\ 
    & \small \textit{Tensor} & \small 32.10 & \small 23.82 & \small 17.00 \\ 
    \bottomrule
    \end{tabular}}
    \vspace{-5pt}
\end{table}

\vspace{-2mm}
\paragraph{\textit{A3}: The gradient scaling factor can even be calculated at the tensor level.} In Table~\ref{tab:granularity}, we compare the pre-training perplexity of our method using different scaling factor granularities. Here, \textit{Channel} indicates that the gradient scaling factor is calculated along the channels with the smaller dimension of each layer, while \textit{Tensor} denotes that a single gradient scaling factor is used for each layer. We keep one-quarter of the original model dimension as the rank. Across model sizes ranging from 60M to 350M, the perplexity difference between these granularities is minimal and both configurations outperform AdamW and GaLore. These results demonstrate that using a tensor-wise scaling factor is sufficient for modest rank training (one-quarter of the original dimension). However, in extreme low-rank scenarios, tensor-wise scaling factor (\namec) outperforms channel-wise ones (\name), as shown in Figure~\ref{fig:svd_rp} (d).

\vspace{-2mm}
\paragraph{\textit{A4}: \name performs better with larger model sizes and more training tokens.} Figure~\ref{fig:train_curve} illustrates the validation perplexity across the training process for LLaMA-350M models. In the early training stages, Fira shows faster convergence and lower perplexity. However, \name gradually catches up, achieving improved performance in the later stages. This observation suggests that AdamW optimization states play a more crucial role in the initial phase (as Fira maintains the low-rank format of these states), while compressing the optimization states into gradient scaling factors (as done in \name) becomes more effective in later stages. Additionally, Figure~\ref{fig:train_curve} indicates that \name seem to benefit from increased training tokens. To quantify this effect, we pre-trained LLaMA-130M models for \{20k, 30k\} steps, with final perplexities for Fira and \name reaching \{22.73, 21.69\} and \{22.84, 21.71\}, respectively, further confirming that \name can gradually catch up Fira with more training tokens. Furthermore, Table~\ref{tab:pre-train} shows that as model size increases, \name demonstrates better scaling capabilities than Fira: validation perplexity decreases from 31.55 to 14.20 when scaling model sizes from 60M to 1B, whereas Fira only improves from 31.06 to 14.31. Overall, \name exhibits superior performance with both larger model sizes and additional training tokens.

\begin{figure}
    \centering
    \includegraphics[width=1 \linewidth]{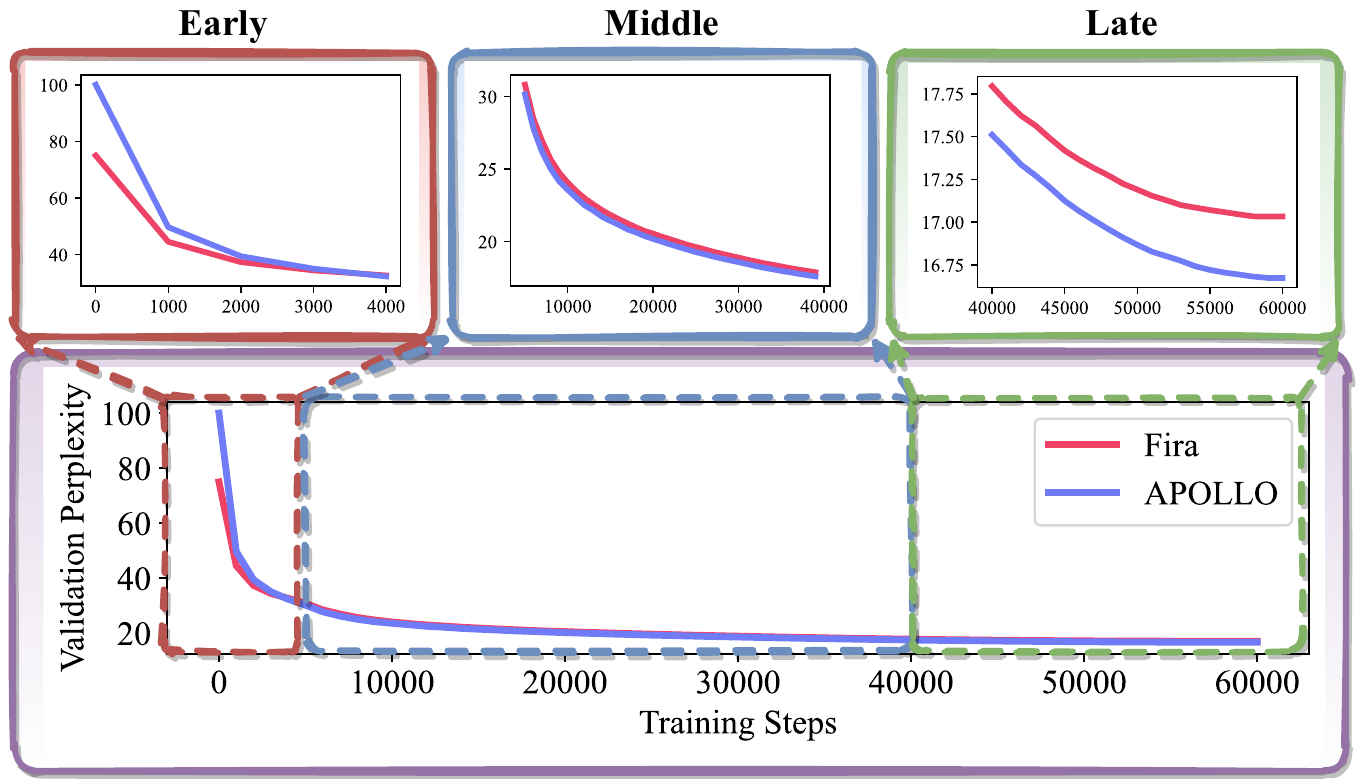}
    \vspace{-4mm}
    \caption{Validation perplexity of pretraining LLaMA-350M on the C4 dataset, with zoomed-in figures showing early, middle, and late stages of training at top, with full training period at bottom.}
    \label{fig:train_curve}
\end{figure}

\vspace{-2mm}
\paragraph{\textit{A5}: \name performs on par with or even better than AdamW in the long-context setting.}
Training LLM with long context windows is computationally expensive, but it is critical to enhance LLM performance by involving more contexts to reason.
Here, we further validate the effectiveness of the \name series on pre-training a LLaMA-350M with a long context window of 1024, four times over original GaLore uses.
To establish a strong baseline, we vary AdamW’s learning rate across a range of [1e-3, 2.5e-3, 5e-3, 7.5e-3, 1e-2]. 
We also lazily tune the scale factor of \name-series by varying \name's in [$\sqrt{1}$, $\sqrt{2}$, $\sqrt{3}$] and \namec's in [$\sqrt{128}, \sqrt{256}, \sqrt{384}$], under a fixed learning rate 1e-2.

As shown in Fig.~\ref{fig:apollo-long-context}, both \name and \namec demonstrate better performance than AdamW while achieving drastic reductions in optimizer memory costs—1/8 or even 1/1024 of AdamW’s memory usage. 
Note that our methods generally exhibit even better performance in the later stage with more training tokens involved, marking it a promising candidate in partial LLM pre-training settings, i.e., long-context window and trillions of training tokens.

\begin{figure}[h]
\centering
\includegraphics[width=0.8 \linewidth]{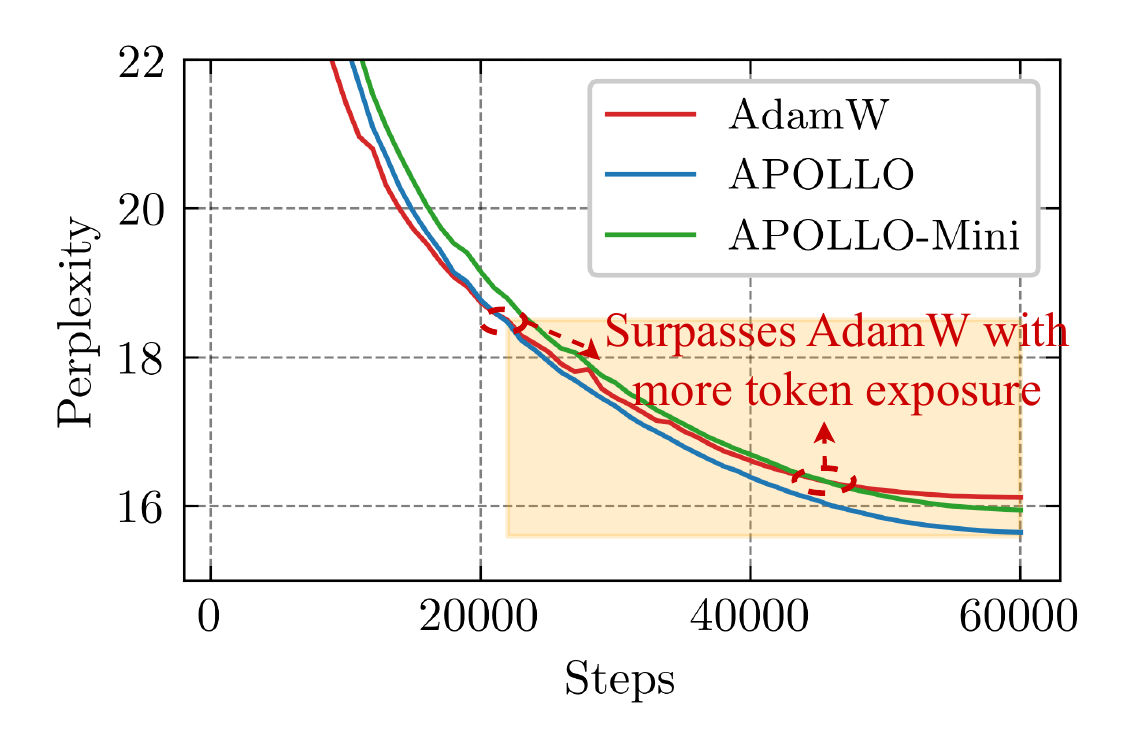}
\vspace{-4mm}
\caption{Perplexity curves of the LLaMA-350M model trained in a long-context window setting. \name and \namec outperform AdamW with a grid-searched learning rate, demonstrating the effectiveness of the \name series in industrial LLM pre-training settings(long sequences and extensive training tokens).}
\label{fig:apollo-long-context}
\vspace{-10pt}
\end{figure}

\subsection{Extra Insights on Why a Stateless Optimizer Can Beat AdamW in Pre-training}
\label{sec:insight}
We provide preliminary insights on why a stateless \name can surpass AdamW in certain scenarios, and we leave a formal one as the future work.

Adam(W) applies 
\(
\tilde{\mathbf{G}}_t = \frac{\mathbf{M}_t}{\sqrt{\mathbf{V}_t} + \epsilon},
\)
which can be viewed as scaling the raw gradient \(\mathbf{G}_t\) by a scaling matrix \(\mathbf{S} = \frac{\tilde{\mathbf{G}}_t}{\mathbf{G}_t}\).
\name observes that this fine-grained, parameter-wise scaling $\mathbf{S}$ can be approximated at the channel or tensor level, validated in Fig.~\ref{fig: loss-motivation}.
Although coarser, this scaling largely preserves the original gradient direction, (e.g., \(\mathbf{G}_t s^R\) in \namec) and thus behaves similarly to SGD. Such an “SGD-like” update depends more on the current gradient and injects greater randomness during training, \textbf{enhancing the ability to escape local optima} and yielding \textbf{better generalization performance}\cite{zhou2020towards, keskar2017improving}. This explains why \name series can surpass AdamW, especially at later training stages (when generalization becomes critical) and for larger models (with more complex loss landscapes).
Key observations supporting this claim include:
\begin{itemize}
    \item In Sec.~\ref{sec:reform_adam}, Fig.~\ref{fig: loss-motivation}, the structured AdamW (\name-style update rule) underperforms AdamW initially but eventually surpasses it.
    \item In Sec.~\ref{sec:ablation} (A4), \name typically outperforms AdamW at later stages of training.
\end{itemize}

\noindent \textbf{Why \name Resembles SGD Yet Performs Well for LLM Training?}
While SGD is associated with stronger generalization, it often struggles with Transformer training~\cite{pan2023toward,zhang2024transformers}. \name reconciles SGD’s generalization benefits with AdamW’s convergence speed, as illustrated by the following two hypotheses~\cite{pan2023toward,zhang2024transformers}.

\begin{table}[htb]
    \centering
    \caption{Directional sharpness comparison across different optimizers.}
    \label{tab:sharpness}
    \resizebox{0.98\linewidth}{!}{
    \begin{tabular}{c|c|c|c|c}
        \toprule
        \small Epoch & \small SGD & \small Adam & \small \name & \small \namec \\
        \midrule
        2 & 1.959722 & 0.009242 & 0.006024 & 0.004017 \\
        5 & 1.512521 & 0.000509  & 0.000249 & 0.000107 \\
        10 & 2.471792 & 0.000242 & 0.000163  & 0.000056 \\
        20 & 3.207535 & 0.000399 & 0.000261 & 0.000101 \\
        \bottomrule
    \end{tabular}}
\end{table}

\paragraph{Hypothesis 1: Directional Sharpness~\cite{pan2023toward}}

A key finding in~\cite{pan2023toward} is that Adam achieves lower \textit{directional sharpness} than SGD, thereby improving Transformer training. The directional sharpness of \(f\) at \(x\) along direction \(v\) (with \(\|v\|_2=1\)) is \(v^\top \nabla^2 f(x) v\).  
Lower directional sharpness implies the possibility of taking larger effective steps, potentially yielding a greater local decrease in the objective. 
In contrast, if the directional sharpness is large, we have no choice but to take a tiny step, as otherwise the loss would blow up due to the second-order term.

Empirical tests on \name/\namec (using a small T5 model for a machine translation task following ~\cite{pan2023toward}) show significantly reduced sharpness relative to SGD and comparable to or better sharpness than Adam(W) (see Table~\ref{tab:sharpness}). This provides a theoretical underpinning for \name’s effectiveness in LLM training.

\paragraph{Hypothesis 2: Adaptive Learning Rates for Transformers~\cite{zhang2024transformers}}
Transformer blocks display varying Hessian spectra, suggesting that block-wise adaptive learning rates are advantageous~\cite{zhang2024transformers}, which can render naive SGD less suitable. However, fully parameter-wise adaptive learning rates (as in AdamW) can be redundant, as shown in Adam-Mini~\cite{zhang2024transformers}, which replaces the second-order moment with group-wise averaging—thereby reducing optimizer memory usage by up to 50\%.

\name applies adaptive learning rates in a structured channel-/tensor-wise manner and goes beyond Adam-Mini by reducing memory usage for both first- and second-order moments, even eliminating optimizer memory in \namec.

\section{Conclusion}

In this paper, we introduced a novel approach for training large language models (LLMs) that strikes an effective balance between memory efficiency and performance. Motivated by the limitations of existing methods like GaLore, which rely on memory-intensive SVD, and inspired by techniques like Fira and Adam-mini, we proposed two methods to achieve structured-wise gradient scaling. Our approach leverages low-rank optimizer states, using random projection only to preserve gradient norms, enabling efficient training without storing the full optimizer state.
Extensive experiments across both pre-training and fine-tuning tasks demonstrate the effectiveness of our \name, consistently surpassing the AdamW baseline with greater memory saving than GaLore.
\namec further squeeze the memory cost by using a rank-1 sub-space, achieving better performance than AdamW at the cost of SGD.
Overall, our method offers a promising solution to the memory bottlenecks in LLM training, providing an efficient alternative that maintains high performance while drastically reducing memory consumption.
\section*{Acknowledgements}

We extend our heartfelt gratitude to the following individuals for their invaluable contributions:

\begin{itemize}
    \item Yuandong Tian, for critical discussions and invaluable insights regarding \name’s optimization behaviors, which greatly informed our preliminary theoretical insights.
\end{itemize}

\nocite{langley00}

\bibliography{ref/ref}
\bibliographystyle{mlsys2025}

\clearpage
\appendix
\section{Appendix}

\subsection{Proof of Gradient Scaling Approximation in Random Projected Low-rank Space}

\subsubsection{Problem Statement}

\paragraph{Notations and Definitions:}
We first introduce the notations and definitions used in the proof:
\begin{itemize}
    \item Let \( \mathbf{G}_t \in \mathbb{R}^{m \times n} \) denote the gradient matrix at iteration \( t \) ($m \leq n$).
    \item Let \( \mathbf{P} \in \mathbb{R}^{r \times m} \) denote the random projection matrix where \( P_{ij} \sim N(0,1/r) \) i.i.d.
    \item Define \( \mathbf{R}_t = \mathbf{P} \mathbf{G}_t \) as the projected gradient.
    \item Let \( \beta_1, \beta_2 \in (0,1) \) be exponential decay rates.
    \item Define \( \mathbf{M}_t \), \( \mathbf{V}_t \) as first and second moments in the original space.
    \item Define \( \mathbf{M}_t^R \), \( \mathbf{V}_t^R \) as first and second moments in projected space.
    \item Let \( T \) denote the total number of iterations.
    \item Let \( n \) denote the number of channels.
    \item Assume zero initialization: \( \mathbf{M}_0 = \mathbf{M}_0^R = 0 \), \( \mathbf{V}_0 = \mathbf{V}_0^R = 0 \).
    \item $\|\|$ indicates $\ell_2$ norm of a vector. 
    \item $\|\|_1$ indicates $\ell_1$ norm of a vector. 
\end{itemize}

\paragraph{Problem:}
We aim to prove that gradient scaling factors $s_j$ and $s_j^R$ in the original and low-rank projected space have a bound for their ratio $s_j^R / s_j$, 

\[
s_j^R / s_j= \frac{\| \mathbf{G}_t[:, j] \|}{\| \tilde{\mathbf{G}}_t[:, j] \|} \frac{\| \tilde{\mathbf{R}}_t[:, j] \|}{\| \mathbf{R}_t[:, j] \|} \\
=\frac{\| \mathbf{G}_t[:, j] \|}{\| \mathbf{R}_t[:, j] \|} \frac{\| \tilde{\mathbf{R}}_t[:, j] \|}{\| \tilde{\mathbf{G}}_t[:, j] \|}
\]
where:
\[
\tilde{\mathbf{R}}_t = \frac{\mathbf{M}_t^R}{\sqrt{\mathbf{V}_t^R}}
\]
and 
\[
\tilde{\mathbf{G}}_t = \frac{\mathbf{M}_t}{\sqrt{\mathbf{V}_t}}
\]

\subsubsection{Norm Preservation Under Random Projection}

\begin{theorem}[Norm Preservation]
For any fixed vector \( x \in \mathbb{R}^m \) and random matrix \( \mathbf{P} \in \mathbb{R}^{r \times m} \) where \( P_{ij} \sim \mathcal{N}(0, 1/r) \) i.i.d., the following holds with high probability:
\[
\Pr[(1-\epsilon)\|x\|^2 \leq \|\mathbf{P}x\|^2 \leq (1+\epsilon)\|x\|^2] \geq 1 - 2\exp\left(-\frac{r\epsilon^2}{8}\right).
\]
\end{theorem}

Theorem A.1 is proven by leveraging the properties of Gaussian random projections and the concentration inequality for the chi-squared distribution.
\begin{proof}
The projected norm \( \|\mathbf{P}x\|^2 \) can be expressed as:
\[
\|\mathbf{P}x\|^2 = \sum_{j=1}^r \left( \sum_{i=1}^m \mathbf{P}_{ji} x_i \right)^2.
\]
Rewriting this using the quadratic form, we have:
\[
\|\mathbf{P}x\|^2 = x^\top \mathbf{P}^\top \mathbf{P} x,
\]
where \( \mathbf{P}^\top \mathbf{P} \) is a symmetric \( m \times m \) matrix. To analyze \( \|\mathbf{P}x\|^2 \), consider the distribution of \( \mathbf{P}^\top \mathbf{P} \).

Each entry of \( \mathbf{P}^\top \mathbf{P} \) is given by:
\[
(\mathbf{P}^\top \mathbf{P})_{ij} = \sum_{k=1}^r \mathbf{P}_{ki} \mathbf{P}_{kj}.
\]
For \( i = j \) (diagonal entries), we have:
\[
(\mathbf{P}^\top \mathbf{P})_{ii} = \sum_{k=1}^r \mathbf{P}_{ki}^2,
\]
and since \( \mathbf{P}_{ki} \sim \mathcal{N}(0, 1/r) \), \( \mathbf{P}_{ki}^2 \sim \text{Exponential}(1/r) \). 

Therefore, \( (\mathbf{P}^\top \mathbf{P})_{ii} \sim \frac{1}{r} \chi^2_r \), where \( \chi^2_r \) is the chi-squared distribution with \( r \) degrees of freedom. For \( i \neq j \) (off-diagonal entries), the expectation is zero:
\[
\mathbb{E}[(\mathbf{P}^\top \mathbf{P})_{ij}] = 0,
\]
due to the independence of \( \mathbf{P}_{ki} \) and \( \mathbf{P}_{kj} \).

The expected value of \( \mathbf{P}^\top \mathbf{P} \) is therefore:
\[
\mathbb{E}[\mathbf{P}^\top \mathbf{P}] = I_m,
\]
where \( I_m \) is the identity matrix.

The expectation of \( \|\mathbf{P}x\|^2 \) is:
\[
\mathbb{E}[\|\mathbf{P}x\|^2] = x^\top \mathbb{E}[P^\top \mathbf{P}] x = x^\top I_m x = \|x\|^2.
\]

Now consider the concentration of \( \|\mathbf{P}x\|^2 \) around its expectation. Define the random variable:
\[
Z = \frac{r \|\mathbf{P}x\|^2}{\|x\|^2}.
\]
Since \( \mathbf{P}_{ij} \) entries are i.i.d., \( Z \sim \chi^2_r \). Using the moment generating function of \( \chi^2_r \), the following concentration bounds can be derived using standard tail inequalities for sub-exponential random variables~\cite{wainwright2015tailbounds}:
\[
\Pr\left(\left| \frac{Z}{r} - 1 \right| \geq \epsilon \right) \leq 2 \exp\left(-\frac{r\epsilon^2}{8}\right).
\]

Returning to \( \|\mathbf{P}x\|^2 \), we scale \( Z \) back:
\[
\|\mathbf{P}x\|^2 = \frac{Z \|x\|^2}{r}.
\]
Thus, with high probability:
\[
(1-\epsilon)\|x\|^2 \leq \|\mathbf{P}x\|^2 \leq (1+\epsilon)\|x\|^2,
\]
and the probability of this event is at least:
\[
1 - 2 \exp\left(-\frac{r\epsilon^2}{8}\right).
\]
\end{proof}

\subsubsection{First Moment Analysis}
\label{subsubsec:first_moment}

\begin{theorem}[First Moment Preservation]
For any channel \( j \),
with probability at least $1 - 2\exp\left(-\frac{r\epsilon^2}{8}\right)$:
\[
(1-\epsilon)\|\mathbf{M}_t[:, j]\|^2 \leq \|\mathbf{M}^R_t[:, j]\|^2 \leq (1+\epsilon)\|\mathbf{M}_t[:, j]\|^2,
\]
using a fixed projection matrix $\mathbb{R}^{r \times m}$ over $t$.
\end{theorem}

\begin{proof}
Our goal is to bound \( \|\mathbf{M}^R_t[:, j]\| \) in terms of \( \|\mathbf{M}_t[:, j]\| \).

\paragraph{Step 1:  Recursive Definitions of $\mathbf{M}_t[:, j]$ and $\mathbf{M}_t^R[:, j]$.}
The first moment \( M_t[:, j] \) in the original space is recursively defined as:
\[
\mathbf{M}_t[:, j] = (1-\beta_1) \sum_{k=0}^{t-1} \beta_1^k \mathbf{G}_{t-k}[:, j],
\]
where \( \mathbf{G}_{t-k}[:, j] \in \mathbb{R}^m \) is the gradient at timestep \( t-k \).

The projected first moment \( \mathbf{M}^R_t[:, j] \) is similarly defined as:
\[
\mathbf{M}^R_t[:, j] = (1-\beta_1) \sum_{k=0}^{t-1} \beta_1^k \mathbf{R}_{t-k}[:, j],
\]
where \( \mathbf{R}_{t-k}[:, j] = \mathbf{P} \mathbf{G}_{t-k}[:, j] \in \mathbb{R}^r \).

\paragraph{Step 2: Projected First Moment in a Lower Dimension.}
With a random matrix \( \mathbf{P} \in \mathbb{R}^{r \times m} \) where \( P_{ij} \sim \mathcal{N}(0, 1/r) \) i.i.d., we have the projected first moment in the low-rank space,
\begin{equation*}
\begin{aligned}
    \mathbf{M}^R_t[:, j] &= (1-\beta_1) \sum_{k=0}^{t-1} \beta_1^k \mathbf{R}_{t-k}[:, j] \\
    &= (1-\beta_1) \sum_{k=0}^{t-1} \beta_1^k \mathbf{P}\mathbf{G}_{t-k}[:, j] \\
    &= \mathbf{P} \bigg( (1-\beta_1) \sum_{k=0}^{t-1} \beta_1^k \mathbf{G}_{t-k}[:, j] \bigg) \\
    &= \mathbf{P} \mathbf{M}_t[:, j]
\end{aligned},
\end{equation*}
by factoring $\mathbf{P}$ out of the summation.

This implies the $\mathbf{M}^R_t[:, j]$ can be viewed as a projected version of $\mathbf{M}_t[:, j]$ into a lower dimension with a fixed $\mathbf{P}$ over time $t$.

\paragraph{Step 3:  Properties of Random Projection}
By Theorem A.1, we have the norm of $\mathbf{M}_t[:, j]$ is preserved in a high probability, 
\begin{equation}
\begin{split}
\Pr\bigg((1-\epsilon)\|\mathbf{M}_{t}[:, j]\|^2 \leq \|\mathbf{M}^R_{t}[:, j]\|^2 \\
\leq (1+\epsilon)\|\mathbf{M}_{t}[:, j]\|^2\bigg) \\
\geq 1 - 2\exp\left(-\frac{r\epsilon^2}{8}\right).
\end{split}
\end{equation}
\noindent\textbf{Remark:}
Here, we assume the projection matrix is fixed over time step $t$.
GaLore~\cite{zhao2024galore} also derives their theorem with the same assumption.
However, as acknowledged in GaLore, using the same projection matrix for the entire training may limit the directions in which the weights can grow. 
Therefore, empirically, as in GaLore, we can periodically resample $\mathbf{P}$ over $T$ iterations to introduce new directions.
Unlike GaLore, which uses time-consuming SVD-based updates, we can simply re-sample $\mathbf{P}$ from the Gaussian distribution by changing the random seed.
\end{proof}

\subsubsection{Second Moment Analysis}
\label{subsubsec:second_moment}

\begin{theorem}[Second Moment Preservation]
For any channel \( j \) and time \( t \), if 
\[
r \geq \frac{8}{\epsilon^2} \log\left(\frac{2t}{\delta}\right),
\]
then with probability at least \( 1 - \delta/2 \):
\[
(1-\epsilon)\|\mathbf{V}_t[:, j]\|_1 \leq \|\mathbf{V}_t^R[:, j]\|_1 \leq (1+\epsilon)\|\mathbf{V}_t[:, j]\|_1
\],
where \( \mathbf{V}_t[:, j] \) and \( \mathbf{V}^R_t[:, j] \) are the second moments in the original and projected spaces, respectively.
\end{theorem}

\begin{proof}
Our goal is to show that the norm of the second moment \( \mathbf{V}_t \) in the original space is preserved under projection to the lower-dimensional space. We proceed by analyzing the recursive definition of \( \mathbf{V}_t \) and applying the results of Theorem A.2 on norm preservation.

\paragraph{Step 1: Recursive Formulation of \( \mathbf{V}_t \)}
The second moment \( \mathbf{V}_t[:, j] \) for channel \( j \) at iteration \( t \) is defined recursively as:
\[
\mathbf{V}_t[:, j] = \beta_2 \mathbf{V}_{t-1}[:, j] + (1 - \beta_2)(\mathbf{G}_t[:, j])^2
\]
By expanding recursively, we can write \( \mathbf{V}_t[:, j] \) as a weighted sum of the squared gradients from all past iterations:
\[
\mathbf{V}_t[:, j] = (1 - \beta_2) \sum_{k=0}^{t-1} \beta_2^k (\mathbf{G}_{t-k}[:, j])^2
\]

\paragraph{Step 2: Projected Second Moment in Lower Dimension}
Similarly, in the projected space, the second moment \( \mathbf{V}_t^R[:, j] \) for channel \( j \) at iteration \( t \) is given by:
\[
\mathbf{V}_t^R[:, j] = \beta_2 \mathbf{V}_{t-1}^R[:, j] + (1 - \beta_2)(\mathbf{R}_t[:, j])^2
\]
Expanding recursively, we have:
\[
\mathbf{V}_t^R[:, j] = (1 - \beta_2) \sum_{k=0}^{t-1} \beta_2^k (\mathbf{R}_{t-k}[:, j])^2
\]

\paragraph{Step 3: $\ell_1$ Norm of Channel-wise Second Moment.}
Then, we can obtain the $\ell_1$ norm of the second-moment term $\mathbf{V}^R_t[:, j]\|_1$
\[
\|\mathbf{V}^R_t[:, j]\|_1 = \sum_{i=1}^{r}(1-\beta_2) \sum_{k=0}^{t-1} \beta_2^k (\mathbf{R}_{t-k}[i, j])^2,
\]

We can swap the summation order and have,
\begin{equation*}
    \begin{aligned}
        \|\mathbf{V}^R_t[:, j]\|_1 &= (1-\beta_2) \sum_{k=0}^{t-1} \beta_2^k \sum_{i=1}^{r}(\mathbf{R}_{t-k}[i, j])^2 \\
        &= (1-\beta_2) \sum_{k=0}^{t-1} \beta_2^k ||\mathbf{R}_{t-k}[:,j]||^2
    \end{aligned}
\end{equation*}

Similarly, we can have
\begin{equation*}
    \begin{aligned}
        \|\mathbf{V}_t[:, j]\|_1 &= (1-\beta_2) \sum_{k=0}^{t-1} \beta_2^k \sum_{i=1}^{n}(\mathbf{G}_{t-k}[i, j])^2 \\
        &= (1-\beta_2) \sum_{k=0}^{t-1} \beta_2^k ||\mathbf{G}_{t-k}[:,j]||^2
    \end{aligned}
\end{equation*}

\paragraph{Step 4: Constructing the Bounds for \( \mathbf{V}_t^R[:, j] \)}
By Theorem A.1, we know that for each \( k \), the $\ell_2$ norm of the projected gradient \( \|\mathbf{R}_{t-k}[:, j]\| \) satisfies:
\[
(1-\epsilon) \|\mathbf{G}_{t-k}[:, j]\|^2 \leq \|\mathbf{R}_{t-k}[:, j]\|^2 \leq 1+\epsilon \|\mathbf{G}_{t-k}[:, j]\|^2,
\]
with probability \( \geq 1 - 2\exp(-r\epsilon^2/8) \).

Therefore, 
\begin{equation*}
    \begin{aligned}
        &\|\mathbf{V}^R_t[:, j]\|_1
        = (1-\beta_2) \sum_{k=0}^{t-1} \beta_2^k ||\mathbf{R}_{t-k}[:,j]||^2 \\
        &\leq (1-\beta_2) \sum_{k=0}^{t-1} \beta_2^k (1+\epsilon) ||G_{t-k}[:,j]||^2 = (1+\epsilon) \|V_t[:, j]\|_1
    \end{aligned}
\end{equation*}

Similarly, we can obtain the lower bound,
\begin{equation*}
    \begin{aligned}
        &\|\mathbf{V}^R_t[:, j]\|_1
        = (1-\beta_2) \sum_{k=0}^{t-1} \beta_2^k ||\mathbf{R}_{t-k}[:,j]||^2 \\
        &\geq (1-\beta_2) \sum_{k=0}^{t-1} \beta_2^k (1-\epsilon) ||G_{t-k}[:,j]||^2 = (1-\epsilon) \|V_t[:, j]\|_1
    \end{aligned}
\end{equation*}

We obtain the following bounds for the $\ell_1$ norm of full projected second moment \( \mathbf{V}_t^R[:, j] \):
\[
\|(1-\epsilon)\mathbf{V}_t[:, j] \|_1 \leq \| \mathbf{V}_t^R[:, j] \|_1 \leq \|(1+\epsilon)\mathbf{V}_t[:, j] \|_1
\].

\noindent \textbf{Step 5: Probability of Success.}
To ensure the bound holds across all \( t \) timesteps, we apply the union bound. For each \( k \), the failure probability is \( 2\exp(-r\epsilon^2/8) \). Across \( t \) timesteps, the total failure probability is:
\[
2t \exp\left(-\frac{r\epsilon^2}{8}\right).
\]

Set this total failure probability to \( \delta/2 \), giving the condition:
\[
r \geq \frac{8}{\epsilon^2} \log\left(\frac{2t}{\delta}\right).
\]

\noindent\textbf{Remark:}
Here, the requirement that \( r \) grows sublinearly as \( \log(t) \) ensures that even for large \( t \), the rank \( r \) does not grow excessively. 
However, empirically, we find our method is not sensitive to rank selection; even a rank of 256 is sufficient to train LLaMA 7B with 150k steps. 
This can be explained by recent Adam-mini~\cite{zhang2024adam} that the variance doesn't need to be precise, and a block-wise approximation is enough, showing that the variance approximation error can be tolerated well.

\end{proof}

\subsubsection{Main Result: Gradient Scaling Approximation}

\label{subsubsec:all}

\begin{theorem}[Main Result]
\label{theorm:main}
For any channel \( j \), with probability \( \geq 1-\delta \):
\[
\frac{\sqrt{1-\epsilon}}{1+\epsilon}\leq \sqrt{\frac{n}{r}}  \frac{s_j^R}{s_j} \leq \frac{\sqrt{1+\epsilon}}{1-\epsilon}
\]
\end{theorem}

\begin{proof}
Express ratio:
\begin{equation*}
    \begin{aligned}
        \frac{s_j^R}{s_j} &=\frac{\| \mathbf{G}_t[:, j] \|}{\| \mathbf{R}_t[:, j] \|} \frac{\| \tilde{\mathbf{R}}_t[:, j] \|}{\| \tilde{\mathbf{G}}_t[:, j] \|} \\
    \end{aligned}
\end{equation*}

Apply Theorem A.2 for the first part, we can obtain the error bound for the first part:
\[
\frac{\| \mathbf{G}_t[:, j] \|}{\| \mathbf{R}_t[:, j]\| } \in [\sqrt{\frac{1}{1+\epsilon}}, \sqrt{\frac{1}{1-\epsilon}}]
\]

For the second part, it is equal to
\begin{equation}
\label{eq:second_part}
\begin{aligned}
    \frac{\| \tilde{\mathbf{R}}_t[:, j] \|^2}{\| \tilde{\mathbf{G}}_t[:, j] \|^2} &= \frac{\|(\frac{\mathbf{M}_t^R}{\sqrt{\mathbf{V}_t^R}})[:, j]\|^2}{\|(\frac{\mathbf{M}_t}{\sqrt{\mathbf{V}_t}})[:, j]\|^2} \\
    &= \frac{\sum_{i=1}^r (\frac{\mathbf{M}_t^R}{\sqrt{\mathbf{V}_t^R}})^2[i, j]}{\sum_{i=1}^n (\frac{\mathbf{M}_t}{\sqrt{\mathbf{V}_t}})^2[i, j]}
\end{aligned}
\end{equation}

\paragraph{SGD with Momentum only}
If we handle SGD with Momentum only, where variance term above is non-existent, and can be simplified as 
\begin{equation*}
\begin{aligned}
    \frac{\| \tilde{\mathbf{R}}_t[:, j] \|^2}{\| \tilde{\mathbf{G}}_t[:, j] \|^2} &= \frac{\|\mathbf{M}_t^R[:, j]\|^2}{\|\mathbf{M}_t[:, j]\|^2} \\
\end{aligned}
\end{equation*}

We can easily apply Theorem A.3 for the first-moment term:
\[
\sqrt{1-\epsilon} \leq \frac{\|\mathbf{M}_t^R[:, j]\|}{\|\mathbf{M}_t[:, j]\|} \leq \sqrt{1+\epsilon}
\]
where the final scaling factor is bounded,

\begin{equation*}
\begin{aligned}
    \frac{\| \tilde{\mathbf{R}}_t[:, j] \|}{\| \tilde{\mathbf{G}}_t[:, j] \|} \in [\sqrt{1-\epsilon}, \sqrt{1+\epsilon}]
\end{aligned}
\end{equation*}

\paragraph{AdamW}
AdamW's case is more tricky, as~\eqref{eq:second_part} involves the element-wise division and cannot easily separate the momentum and variance.
However, recent works such as Adam-mini~\cite{zhang2024adam} and GaLore-mini~\cite{huang2024galore} find out that the variance term can be approximated as an average of a block-wise (original full-rank space) or channel-wise (projected low-rank space).
Given the $\ell_1$ norm of the variance term is bounded based on Theorem A.4, we take this assumption by replacing the variance term as the average of variance vector, i.e.,  $\frac{||\mathbf{V}_t[:, j]||_1}{n}$ and $\frac{||\mathbf{V}_t^R[:, j]||_1}{r}$ in~\eqref{eq:second_part}.
Then it is approximated as,
\begin{equation*}
\begin{aligned}
    \frac{\| \tilde{\mathbf{R}}_t[:, j] \|^2}{\| \tilde{\mathbf{G}}_t[:, j] \|^2} &= \frac{\sum_{i=1}^r (\frac{\mathbf{M}_t^R[i, j]^2}{\frac{||\mathbf{V}_t^R[:, j]||_1}{r}})}{\sum_{i=1}^n (\frac{\mathbf{M}_t[i, j]^2}{\frac{||\mathbf{V}_t[:, j]||_1}{n}})} \\
    &= (\frac{r}{n}) \frac{||\mathbf{V}_t[:, j]||_1}{||\mathbf{V}_t^R[:, j]||_1} \frac{\|\mathbf{M}_t^R[:, j]\|^2}{\|\mathbf{M}_t[:, j]\|^2}
\end{aligned}
\end{equation*}

Multiply inequalities from theorem A.3 and theorem A.4 with union bound probability \( \geq 1-\delta \), we have the above term
\begin{equation*}
\begin{aligned}
    \sqrt{\frac{n}{r}}\frac{\| \tilde{\mathbf{R}}_t[:, j] \|}{\| \tilde{\mathbf{G}}_t[:, j] \|}  \in [\sqrt{\frac{1-\epsilon}{1+\epsilon}}, \sqrt{\frac{1+\epsilon}{1-\epsilon}}]
\end{aligned}
\end{equation*}

Then, we have the bounded ratio,
\begin{equation*}
    \begin{aligned}
        \sqrt{\frac{n}{r}} \frac{s_j^R}{s_j} &= \sqrt{\frac{n}{r}}\frac{\| \mathbf{G}_t[:, j] \|}{\| \mathbf{R}_t[:, j] \|} \frac{\| \tilde{\mathbf{R}}_t[:, j] \|}{\| \tilde{\mathbf{G}}_t[:, j] \|} \in  [\frac{\sqrt{1-\epsilon}}{1+\epsilon}, \frac{\sqrt{1+\epsilon}}{1-\epsilon}]
    \end{aligned}
\end{equation*}

\noindent\textbf{Remark:}
This contains the constant factor $\sqrt{\frac{n}{r}}$, suggesting we should scale the gradient to make sure it has consistent behavior as AdamW with structured learning rate update.
This gradient scale factor can be combined with the learning rate.
When the $r$ is too small compared to $n$, as in our \namec case, which uses rank-1 space, we specifically assign the scaling factor by using $\sqrt{128}$.

\noindent\textbf{Probability of Success:}
We now establish the probability of success. Both Theorem A.3 and Theorem A.4 rely on the same random projection matrix \( P \) are derived from Theorem A.2 (norm preservation for random projections). Therefore, the probability of failure for both bounds is governed by the failure of Theorem A.2.

For a single timestep \( t \), the failure probability of Theorem A.2 is:
\[
\Pr(\text{Theorem A.2 fails at timestep } t) \leq 2\exp\left(-\frac{r\epsilon^2}{8}\right).
\]

Across all \( t \) timesteps, the total failure probability (union bound) is:
\[
\Pr(\text{Theorem A.2 fails for any timestep}) \leq 2t\exp\left(-\frac{r\epsilon^2}{8}\right).
\]

Set this total failure probability to \( \delta \):
\[
2t\exp\left(-\frac{r\epsilon^2}{8}\right) \leq \delta.
\]

Solving for \( r \), we require:
\[
r \geq \frac{8}{\epsilon^2} \log\left(\frac{2t}{\delta}\right).
\]

This ensures that both Theorem A.3 and Theorem A.4 hold simultaneously with probability \( \geq 1 - \delta \), which together make Theorem A.5 hold.

\end{proof}

\subsection{Empirical validation of the derived bound in ~Theorem \ref{theorm:main}}
\label{subsec:empirical_evidence}
In this part, we present a visualization of the scaling factor ratio \(\sqrt{n/r}\) derived in Theorem \ref{theorm:main}. The plot demonstrates how the ratio adheres to the theoretical bounds under various rank settings, providing empirical support for the theorem.

Here, we consider the following variants:
\begin{itemize}
    \item \textbf{AdamW with the same structured channel-wise learning rate adaptation rule:} This variant uses a full rank \(n\) and serves as the golden standard for estimating \(s_j\), the scaling factor.
    \item \textbf{APOLLO with rank \(r\):} This variant computes a low-rank approximated version of the scaling factor, \(s_j^R\), which should theoretically be \(\sqrt{n/r}\) times smaller than \(s_j\).
\end{itemize}

We visualize the channel-wise scaling factor on the LLaMA-350M model~\footnote{To ensure consistent optimization trajectories across the variants, we use the same learning rate as APOLLO with rank \(1/4n\). Additionally, we scale the final gradient updates using the heuristic ratio derived from the rank settings relative to \(1/4n\).}, comparing APOLLO with ranks \(1/8n\) and \(1/4n\). These configurations should yield scaling factor ratios of approximately \(\sqrt{1/8}\) (\(\sim 0.35\)) and \(1/2\), respectively, relative to the full-rank AdamW.

As shown in Fig. \ref{fig:scaling_factor_comparison}, the scaling factor ratio adheres closely to the theoretical predictions across different layer types (e.g., attention, MLP) and model stages (e.g., early, middle, or late layers).

\begin{figure*}[ht]
    \centering
    \includegraphics[width=\linewidth]{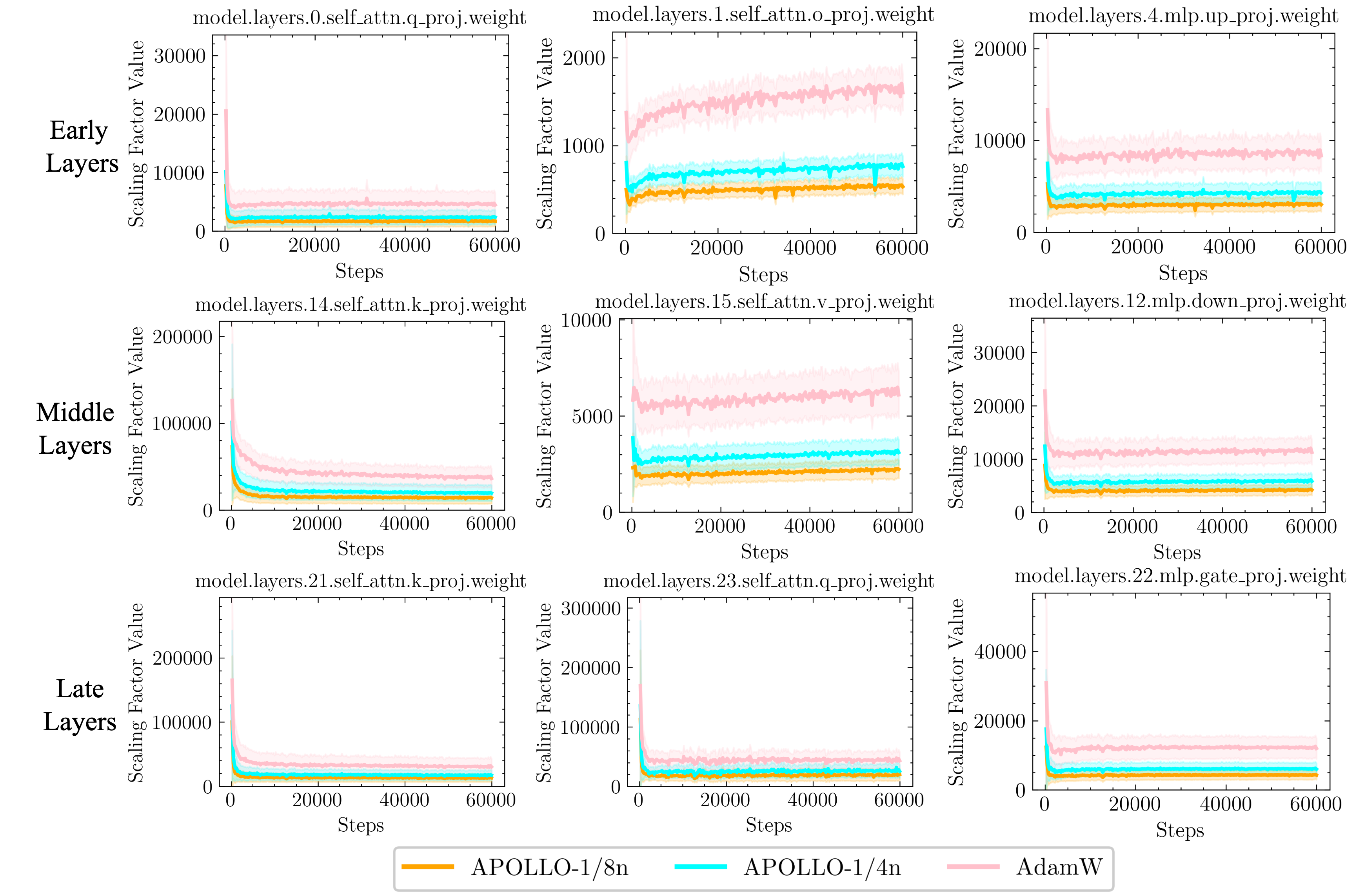}
    \caption{Visualization of the channel-wise scaling factor ratio for APOLLO with rank \(1/8 n\) and \(1/4 n\), compared with AdamW (full rank \(n\)). The empirical data aligns well with the theoretical ratios \(1 : \sqrt{2} : 2\sqrt{2}\), validating the bounds across various layer types and stages on the LLaMA-350M model.}
    \label{fig:scaling_factor_comparison}
\end{figure*}

\subsection{Training throughput of GaLore-type Optimizer on LLaMA-1B}

We further show the training throughput for Galore-type low-rank optimizer (Galore, Fira) in Fig.~\ref{fig:fira:throuput}.
At every 200 update step, they need to call SVD to update the projection matrix, leading to a drastic drop in training throughput.
Although Galore tries to amortize the cost by relaxing the update gap, the significantly high cost is hard to amortize fully as we still keep a short update gap to keep performance; for example, to update the projection matrix for a LLaMA 7B model needs 10 mins, while inference takes seconds.

\begin{figure}
    \centering
    \includegraphics[width=0.9\linewidth]{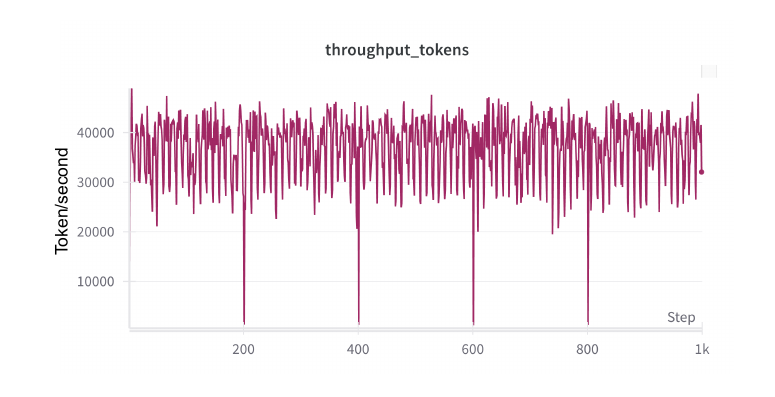}
    \caption{The training throughput of Galore-type low-rank optimizer with many spikes due to the expensive SVD operation every 200 steps.}
    \label{fig:fira:throuput}
\end{figure}

\subsection{Detailed Pre-Training Setting}
\label{sub:setting}
This section provides an overview of the LLaMA architectures and the hyperparameters used during pre-training. To ensure a fair comparison, we adopt the same settings as \citet{zhao2024galore}. Table~\ref{Pre_setting} outlines the hyperparameters for the various LLaMA model sizes. Across all architectures, we use a maximum sequence length of 256 and a batch size of 131K tokens. Additionally, we apply a learning rate warm-up over the first 10\% of training steps, followed by a cosine annealing schedule that gradually reduces the learning rate to 10\% of its initial value.
\begin{table*}[ht]
\centering
\caption{Hyper-parameters of LLaMA architectures for pre-training.}
\vspace{11pt}
\label{Pre_setting}
\resizebox{0.6\linewidth}{!}{
\begin{tabular}{@{}ccccccc@{}}
\toprule
Params & Hidden & Intermediate & Heads & Layers & Steps & Data Amount (Tokens) \\ \midrule
60M    & 512    & 1376         & 8     & 8      & 10K   & 1.3 B       \\
130M   & 768    & 2048         & 12    & 12     & 20K   & 2.6 B       \\
350M   & 1024   & 2736         & 16    & 24     & 60K   & 7.8 B       \\
1 B    & 2048   & 5461         & 24    & 32     & 100K  & 13.1 B      \\
7 B    & 4096   & 11008        & 32    & 32     & 150K  & 19.7 B      \\ \bottomrule
\end{tabular}
}
\end{table*}

\name runs using the same learning rate 0.01 and a subspace change frequency \textit{T} of 200 without tuning, following the Galore open-sourced settings.
The scale factor \( \alpha \) is considered a fractional learning rate, which is set to 1 by default in \name for models with a size of less than 1B, showing our method doesn't need too much tuning like Galore and Fira.
On 1B-model, we set the high-rank \name with a $\alpha=\sqrt{1/2}$ and the high-rank \name w SVD with a $\alpha=0.4$.
As we find the scaling factor increases with the rank $r$, therefore we scale the gradient factor in \namec with setting \(\alpha\) to $\sqrt{128}$.

\subsection{Detailed Fine-Tuning Setting}
\label{sub:setting}

\subsubsection{Commonsense reasoning fine-tuning}

We use the implementation from~\cite{liu2024dora} with the chosen hyperparameters detailed in Table~\ref{tab:llama_commonsense_hyperparameters}.

\subsubsection{MMLU fine-tuning}

We use the implementation from~\cite{zheng2024llamafactory}.
We adopt the implementation from~\cite{zheng2024llamafactory}. For a fair and comprehensive comparison, we set the rank to 8 and sweep the learning rate across the range [5e-6, 7.5-6, 1e-5, 2.5e-5, 5e-5, 7.5e-5, 1e-4, 1.5e-4, 2e-4] for GaLore, Fira, \name, and \namec. Specifically, \namec uses a scaling factor of $\sqrt{4}$ for fine-tuning LLaMA-3-8B and Mistral-7B, while a factor of $1$ is applied to Gemma-7B, as it exhibits higher sensitivity during fine-tuning. The full fine-tuning and LoRA results are taken from~\cite{zhang2024q}.

\begin{table*}[h]
\centering
\caption{Hyperparameter of Llama-3.2-1B on the commonsense reasoning tasks.}
\vskip 0.1in
\small
\begin{tabular}{ccccccccc}
\toprule
\textbf{Hyperparameters} & AdamW & LoRA & DoRA & Galore & Fira & \name w.SVD & \name & \namec\\ \midrule
Rank r & - & 32 & 32 & 32 & 32 & 32 & 32 & 1 \\ 
$\alpha$ & - & 64 & 64 & - & - & - & - & - \\ 
scale & - & - & - & 0.25 & 0.25 & 1.0 & $\sqrt{5}$& $\sqrt{128}$\\
Dropout & \multicolumn{8}{c}{0.05} \\
LR & [2e-5, 5e-5] & 3e-4 & 3e-4 & 3e-4 & 3e-4 & 3e-4 & 3e-4 & 3e-4  \\
LR Scheduler & \multicolumn{8}{c}{Linear} \\
Batch size & \multicolumn{8}{c}{32} \\
Warmup Steps & \multicolumn{8}{c}{100} \\
Epochs & \multicolumn{8}{c}{3} \\
Where & \multicolumn{8}{c}{Q,K,V,Up,Down} \\
\bottomrule
\end{tabular}
\label{tab:llama_commonsense_hyperparameters}
\vskip -0.1in
\end{table*}

%

\end{document}